\documentclass[11pt,english]{article}
\usepackage[utf8]{inputenc}
\usepackage{amsmath}
\usepackage{amsthm}
\usepackage{amssymb}
\usepackage{amsfonts}
\usepackage{comment}
\usepackage{mathtools}
\usepackage{hyperref}
\usepackage{algorithm}
\usepackage{algorithmic}
\usepackage{graphicx}
\usepackage[font=footnotesize]{caption}
\usepackage{xcolor}
\usepackage{enumitem}
\usepackage{nicefrac}
\DeclareMathOperator*{\argmin}{arg\,min}

\usepackage{a4wide}
\usepackage{chngpage}
\usepackage{authblk}

\newtheorem{assumption}{Assumption}
\newtheorem{remark}{Remark}
\newtheorem{theorem}{Theorem}

\newtheorem{lemma}{Lemma}
\newtheorem{definition}{Definition}

\usepackage[T1]{fontenc}
\usepackage{babel}
\usepackage{threeparttable}
\usepackage{booktabs}
\usepackage{etoolbox}
\appto\TPTnoteSettings{\footnotesize}

\author[1]{Aleksandar Armacki}
\author[2]{Dragana Bajovic}
\author[3]{Dusan Jakovetic}
\author[1]{Soummya Kar}
\affil[1]{Carnegie Mellon University, Pittsburgh, PA, USA\\ \texttt{\{aarmacki,soummyak\}@andrew.cmu.edu }}
\affil[2]{Faculty of Technical Sciences, University of Novi Sad, Novi Sad\\ \texttt{dbajovic@uns.ac.rs}}
\affil[3]{Faculty of Sciences, University of Novi Sad, Novi Sad\\ \texttt{dusan.jakovetic@dmi.uns.ac.rs}}

\title{A One-shot Framework for Distributed Clustered Learning in Heterogeneous Environments}
\date{}

\begin{document}

\maketitle

\begin{abstract}
  The paper proposes a family of communication efficient methods for distributed learning in heterogeneous environments in which users obtain data from one of $K$ different distributions. In the proposed setup, the grouping of users (based on the data distributions they sample), as well as the underlying statistical properties of the distributions, are apriori unknown. A family of One-shot Distributed Clustered Learning methods (ODCL-$\mathcal{C}$) is proposed, parametrized by the set of admissible clustering algorithms $\mathcal{C}$, with the objective of learning the true model at each user. The admissible clustering methods include $K$-means (KM) and convex clustering (CC), giving rise to various one-shot methods within the proposed family, such as ODCL-KM and ODCL-CC. The proposed one-shot approach, based on local computations at the users and a clustering based aggregation step at the server is shown to provide strong learning guarantees. In particular, for strongly convex problems it is shown that, as long as the number of data points per user is above a threshold, the proposed approach achieves order-optimal mean-squared error (MSE) rates in terms of the sample size. An explicit characterization of the threshold is provided in terms of problem parameters. The trade-offs with respect to selecting various clustering methods (ODCL-CC, ODCL-KM) are discussed and significant improvements over state-of-the-art are demonstrated. Numerical experiments illustrate the findings and corroborate the performance of the proposed methods.
\end{abstract}

\section{Introduction}

Distributed learning (DL) is a paradigm where many users collaborate, with the aim of jointly training a model. Each user holds a private dataset and the training process is either coordinated by a central server~\cite{zhang-duchi,pmlr-v54-mcmahan17a}, or done in a peer-to-peer (i.e., decentralized) manner~\cite{kar-link_failure,kar-link_failure2,jakovetic-unification,jakovetic-fast}. While such an approach helps alleviate the storage and computation burden for any single user, it imposes significant communication costs on the system as a whole~\cite{comm-comp,konevcny2016federated,pmlr-v124-mishchenko20a,hong_distributed-features,poor-overview,comm-admm}. In this paper we focus on the setting where a central server coordinates the training and in the reminder of the paper we will refer to the client-server setup as distributed. A specific instance of DL that has attracted a lot of interest recently is federated learning (FL)~\cite{pmlr-v54-mcmahan17a}. FL differs from the standard DL approaches in that only a subset of users is active in each training round and users perform multiple local computation rounds before synchronizing, trading computation for communication. Another issue associated with FL comes from statistical heterogeneity, as users often contain datasets generated by different distributions, making the system as a whole highly heterogeneous. A global model can therefore be very bad for an individual user~\cite{wang2020tackling,yu2020salvaging,bedi2023fedbc}.

\begin{table*}[t]
\caption{Comparison of methods for distributed clustered learning. We use the order notation to hide problem related constants, instead highlighting the dependence on the number of available samples $n$, size of cluster $|C_k|$, total number of users and clusters, $m$ and $K$, and the separation $D$. The notation $|C_{(i)}|$ stands for the size of the $i$-th largest cluster. Methods prefixed by ``I'' allow for solving the local problems inexactly, up to precision $\varepsilon$. Acronyms: CR - Communication Rounds, SR - Sample Requirement (number of samples required for the guarantees to hold), IC - Initialization Conditions (sufficiently good initialization), CK - Cluster Knowledge (knowing the true cluster structure), $K$ - Knowledge of $K$, SS - Sufficient Separation ($D$ sufficiently large). We use $\surd$ to denote that the method does not require a condition, while $\times$ means that the method requires a condition (e.g., IFCA requires sufficient separation). }
\label{tab:conver}
\begin{adjustwidth}{-1in}{-1in} 
\begin{center}
\begin{threeparttable}
\begin{small}
\begin{sc}
\begin{tabular}{lccccccr}
\toprule
Method &  CR & SR & IC & CK & $K$ & SS \\
\midrule
IFCA~\cite{ieee_ghosh}    & $\mathcal{O}\left(\frac{m}{|C_{(K)}|}\log\left(\frac{D^2n|C_{(K)}|^5}{K^2m^4}\right)\right)$ & $\Omega\left(\max\left\{|C_{(1)}|,\frac{K}{D^4}\right\}\right)$ & $\times$ & $\surd$ & $\times$ & $\times$\\
ALL-for-ALL~\cite{even2022sample}    & $\Theta\left(\frac{nm}{K}\log\left(\frac{nm}{K}\right) \right)$  & $\Omega\left(\log\left(\frac{nm}{K} \right)\right)$ & $\surd$ & $\times$\tnote{*} & $\surd$ & $\surd$ \\
ODCL-KM (This paper)   & 1 & $\Omega\left(\max\left\{|C_{(1)}|, \frac{(|C_{(K)}| + \sqrt{m})^2}{|C_{(K)}|D^2} \right\} \right)$ & $\surd$ & $\surd$ & $\times$ & $\surd$ \\
ODCL-CC (This paper) & 1 & $\Omega\left(\max\left\{|C_{(1)}|, \frac{(m - |C_{(K)}|)^2}{|C_{(K)}|^2D^2} \right\} \right)$ & $\surd$ & $\surd$ & $\surd$ & $\surd$ \\
I-ODCL-KM (This paper)   & 1 & $\Omega\left(\max\left\{|C_{(1)}|, \left(\frac{|C_{(K)}|^2D^2}{(\sqrt{|C_{(K)}|} + \sqrt{m})^2} - \varepsilon\right)^{-1} \right\} \right)$ & $\surd$ & $\surd$ & $\times$ & $\surd$ \\
I-ODCL-CC (This paper) & 1 & $\Omega\left(\max\left\{|C_{(1)}|, \left(\frac{|C_{(K)}|^2D^2}{(m - |C_{(K)}|)^2} - \varepsilon\right)^{-1} \right\} \right)$ & $\surd$ & $\surd$ & $\surd$ & $\surd$ \\
\bottomrule
\end{tabular}
\end{sc}
\end{small}
\begin{tablenotes}
    \item[*] The authors provide a method for estimating cluster structure as a preprocessing step, that only applies to quadratic losses, under strong assumptions.
\end{tablenotes}
\end{threeparttable}
\end{center}
\vskip -0.1in
\end{adjustwidth}
\end{table*}

Many different approaches addressing the shortcomings of the global model have been proposed. One such approach is clustered learning (CL), that arises naturally in many applications. For example, consider a speech recognition problem. Speech recognition models are known to be sensitive to accents and dialects, e.g.,~\cite{speech1,speech2}, with software such as Alexa or Youtube's closed caption struggling with, e.g., Irish and Scottish accents. While English is spoken in both countries, the different accents necessitate different models. Another example is that of next word prediction considered in~\cite{sattler}, where users jointly train a language model for next-word prediction from text messages. The statistics of users' messages varies based on demographic factors, interests, etc. For example, messages composed by teenagers exhibit different statistics than those composed by elderly people. Both examples illustrate the need for \emph{personalized models}, catering to the needs of individual users. While desirable, complete personalization is not achievable in many applications, due to the sheer number of users and computational cost it incurs. CL offers \emph{partial} personalization, where users' models are not completely individualized, but the similarity of users belonging to the same group is exploited to create better models. The underlying assumption in CL is the presence of $K$ different data distributions $\mathcal{D}_k$, with each user sampling their data from only one of them. This leads to a natural clustering of users, given by $C_k = \{i \in [m]: \: \text{user $i$'s data follows distribution $\mathcal{D}_k$} \}$. Since each user samples data from only one distribution, it follows that $\cup_{k \in [K]}C_k = [m]$ and $C_k \cap C_j = \emptyset$. In such a scenario the goal is to learn $K$ models associated with the underlying clusters, so that users belonging to the same clusters have the same models. Allowing for $1 \leq K \leq m$, CL can again be seen as an intermediary between the global and local learning, with $K=1$ resulting in a global model, while $K = m$ resulting in purely local models. A key aspect of this approach is computational efficiency compared to fully personalized learning, as it can significantly decrease the number of models produced when $K \ll m$. Many works assume a clustered structure among users, e.g.,~\cite{ieee_ghosh,even2022sample,mansour2020three,ghosh2019robust,Sattler_clustered,armacki2022personalized}.  

The problem this work aims to address has multiple layers. Firstly, note that the grouping of users is not known apriori, hence, is it possible to simultaneously learn the models and the grouping of the users? Secondly, modern DL systems often consist of millions of devices, communicating millions of parameters,~\cite{fed_mobile_edge}. Can learning be facilitated in such scenarios, where more than a few communication rounds is impossible? Thirdly, is it possible to guarantee good performance of a method that works under all of the said constraints? We answer these questions positively in Theorem~\ref{thm:main}, Section~\ref{subsec:main}, and give the exact conditions, in the form of requirements on the sample size, under which this is possible. Finally, is it possible to trade the sample requirements, established in Theorem~\ref{thm:main}, for computation cost or MSE accuracy? This question is answered positively, in Theorems~\ref{thm:inexact_erm} and~\ref{thm:inexact_clust} (Sections~\ref{subsec:inexact} and~\ref{subsec:inexact_clust}).

To address the many layers of the problem considered in this work, we develop a family of one-shot methods, dubbed ODCL-$\mathcal{C}$, that requires a single communication round. The main advantage of the proposed family is the immense communication saving achieved, while offering personalization and maintaining order-optimal generalization guarantees. Note that the conventional iterative methods, such as FedAvg~\cite{pmlr-v54-mcmahan17a}, have very poor performance guarantees when limited to a few communication rounds. Therefore, the proposed family offers a prime candidate in scenarios of extreme communication constraints, as it is guaranteed to both require a single communication round and to generalize well, given a sufficient number of local samples. Our framework offers personalization to users, by producing multiple models based on an inferred grouping of the users. Another advantage is the ease of implementation of the method, as discussed in Section~\ref{sec:method} ahead. Moreover, due to the flexibility of choice of the clustering algorithm, we provide a family of easy to implement personalization algorithms with enormous communication savings and strong theoretical guarantees.

In Table~\ref{tab:conver} we compare different methods for CL along several dimensions. Compared to existing methods~\cite{ieee_ghosh,even2022sample}, the ODCL-$\mathcal{C}$ family reduces communication cost by a factor of $\mathcal{O}\left(\frac{\kappa m}{|C_{(K)}|}\log\left(\frac{2D}{\varepsilon}\right)\right)$, where $\kappa = \frac{L}{\mu}$ is the condition number, $m$ is the total number of users, $|C_{(K)}|$ is the size of the smallest cluster, $D$ is the minimal distance between different population optimal models (see Assumption~\ref{asmpt:clusters} ahead), while $\varepsilon$ is the final accuracy threshold (see Section~\ref{subsec:comparison} ahead). Even though ODCL-$\mathcal{C}$ provides significant communication savings, it still achieves the order-optimal MSE and has comparable sample requirements to the state-of-the-art methods in~\cite{ieee_ghosh,even2022sample}. Comparing the additional requirements, we see that ODCL-CC is the most flexible algorithm, requiring no prior knowledge, strong separation or good initialization. ODCL-KM requires only knowledge of the true number of clusters $K$, same requirement made by IFCA~\cite{ieee_ghosh} and a weaker requirement compared to the true underlying clustering structure, made by~\cite{even2022sample}. Additionally, the proposed family requires no initialization conditions, unlike IFCA, that requires sufficiently good initialization for the guarantees to hold. A more detailed comparison with order-optimal methods~\cite{ieee_ghosh} and~\cite{even2022sample} is presented in Section~\ref{subsec:comparison}, while a more detailed literature review is provided after the contributions.

\textbf{Contributions.} The main contributions of the paper are:

$1)$ We propose a family of one-shot methods for distributed CL under the presence of statistical heterogeneity, termed ODCL-$\mathcal{C}$. The family is parametrized by a class of admissible clustering algorithms $\mathcal{C}$, allowing for adaptability to different scenarios (e.g., whether the true number of clusters $K$ is available). The choice of different clustering algorithms gives rise to new methods, e.g., using $K$-means, convex clustering and $K$-means++ gives rise to ODCL-KM, ODCL-CC and ODCL-KM++, respectively. Allowing for the local problems to be solved inexactly gives rise to inexact methods, such as I-ODCL-KM, I-ODCL-CC, I-ODCL-KM++. Due to working in a client-server setup, ODCL-$\mathcal{C}$ is applicable to FL.

$2)$ For strongly convex costs, ODCL-$\mathcal{C}$ is shown to achieve order-optimal MSE guarantees in terms of sample size, i.e., matches the order-optimal MSE rates of centralized learning, provided that each user has a number of data points above a threshold. This establishes regimes in which a single communication round suffices for order-optimality. Moreover, the proposed methods remove strong requirements made by state-of-the-art methods, such as initialization sufficiently close to the population optima and strong population separation, made by~\cite{ieee_ghosh}, and knowledge of the cluster structure, made by~\cite{even2022sample}. Our family is applicable even when the number of clusters is unknown, via ODCL-CC.

$3)$ We develop a new and unified probabilistic analysis of clustering methods, showing that the true clusters are recovered with high probability. To do so, we leverage the interplay of the quality of locally produced models and (deterministic) recovery guarantees of clustering algorithms. This requires tying in tail probability bounds for generalization to recovery of clusters, which is of independent interest.

$4)$ We explicitly derive the expression for the threshold on the number of data points for the users to achieve order-optimal MSE rate and show how it depends on various system parameters, e.g., size of clusters, difficulty of the clustering problem. We specialize the sample requirement for ODCL-CC and ODCL-KM, highlighting how different methods from the family result in different sample requirements.

$5)$ Finally, we verify our theoretical findings via numerical experiments on both synthetic and real data. On synthetic data we show that the proposed methods achieve order-optimal MSE rate and match the performance of oracle methods that know the true cluster membership. On real data we show that our methods achieve test accuracy comparable to oracle methods, regardless of initialization, whereas the performance of~\cite{ieee_ghosh} strongly depends on the quality of initialization.

\textbf{Literature review.} We next review the related literature, in particular, one-shot and CL methods, as well as a closely related learning problem, namely that of distributed estimation. 

\emph{One-shot methods} have been studied in~\cite{zhang-duchi,guha2019one,zhou2020distilled,one-shot_fl,one_shot_clust}. The work~\cite{zhang-duchi} studies one-shot averaging, where each user trains a model locally and the final model is produced by averaging all the local models. The method achieves the same MSE guarantees as centralized learning, i.e., order-optimal rates in terms of sample size, provided that the number of samples available to each user is above a threshold. The work~\cite{guha2019one} proposes to train $K $ ensemble based methods for supervised and semi-supervised problems. The work~\cite{zhou2020distilled} proposes a one-shot distillation method. The work~\cite{one-shot_fl} studies one-shot methods in distributed settings under constraints on the communication budget and shows order-optimality under certain regimes. The work~\cite{one_shot_clust} introduces a one-shot FL method designed specifically for data clustering. The methods in~\cite{zhang-duchi,one-shot_fl} provide strong theoretical guarantees, however, they assume that the data across all users follows a single distribution. Note that our use of term \emph{one-shot} refers to the number of \emph{communication rounds}, which is somewhat different from the use in \emph{few-shot learning} where the goal is to learn a model with a limited \emph{number of samples}, e.g.,~\cite{Tao_2022,WANG2021107510,10.1145/3386252}. In modern DL systems, particularly in FL, data across different users comes from different distributions, hence violating the single distribution assumptions made in prior works. \emph{To the best of our knowledge, no theoretical results for one-shot methods are established under the presence of statistical heterogeneity, i.e., multiple data distributions}\footnote{\cite{one_shot_clust} provide theoretical analysis under heterogeneity, but the method is designed only for data clustering and not general learning.}. 

\emph{CL} has been studied in~\cite{ieee_ghosh,even2022sample,mansour2020three,ghosh2019robust,Sattler_clustered,armacki2022personalized}. Works~\cite{ieee_ghosh,mansour2020three} propose similar methods, that iteratively estimate cluster membership and perform model updates, with order-optimal rates in terms of sample size. The method in~\cite{ghosh2019robust} is a robust algorithm for clustered learning, under the presence of adversarial users. The method in~\cite{Sattler_clustered} is based on successive bi-partitioning of the users. The method in~\cite{armacki2022personalized} aims to simultaneously infer the clustering of users and train models, without requiring knowledge of the number of clusters. Finally,~\cite{even2022sample} study oracle complexity lower-bounds and propose an algorithm that matches the lower bound. All of the methods require many rounds of communication and model training, among which, the best is achieved by~\cite{ieee_ghosh}, requiring $\mathcal{O}\left(\frac{\kappa m}{|C_{(K)}|}\log\left( \frac{2D}{\varepsilon}\right) \right)$ communication rounds.

\emph{Distributed estimation} (DE) is a well-studied problem, where the goal is to estimate a common global model parameter based on data held at different users. It is studied both in decentralized networks, e.g.,~\cite{robust_LMS,diffusion_LMS,diffusion_LMS1,diffusion_LMS2,sensor_attacks,kar_estimaton}, and distributed ones, e.g.,~\cite{distributed-estimation}. Work~\cite{robust_LMS} proposes an algorithm for robust least-mean squares (LMS) estimation over a network, while~\cite{diffusion_LMS,diffusion_LMS1,diffusion_LMS2} study diffusion algorithms for LMS over adaptive networks. Work~\cite{sensor_attacks} studies DE under sensor attacks, while~\cite{kar_estimaton} studies DE under non-linear observations and imperfect communication. Finally,~\cite{distributed-estimation} studies DE with reduced-dimensionality sensor observations. The classical linear DE problem with squared norm loss, i.e., $\ell(\theta;x_i) = \frac{1}{2}\|A_i\theta - b_i \|^2$, where $x_i = (A_i,b_i)$ is the data available to user $i$, is a special case of the more general learning problem considered in this paper, with $\ell$ being a general convex function (see Assumption~\ref{asmpt:loss_function} ahead; see also~\cite{vlaski_et_al} for a discussion along these lines). Moreover, the classical DE problem assumes \emph{homogeneity}, i.e., each user samples from the same distribution and the goal is to estimate a single global model, while we assume \emph{heterogeneity}, with the goal of estimating $K$ models corresponding to $K$ distributions. This entails inferring from which distribution each user samples their data, making the problem more complex.
 
\textbf{Paper Organization.} The rest of the paper is organized as follows. Section~\ref{sec:problem} formally defines the problem. Section~\ref{sec:method} describes the proposed family. Section~\ref{sec:theory} presents the main results of the paper. Section~\ref{sec:numerical} presents numerical results. Finally, Section~\ref{sec:conclusion} concludes the paper. Appendix contains the proofs ommited from the main body. The reminder of the section introduces the notation used throughout the paper.

\textbf{Notation.} The $d$-dimensional vector space of real numbers is denoted by $\mathbb{R}^d$. $\mathbb{N}$ denotes the set of positive integers. $\langle \cdot, \cdot \rangle$ denotes the Euclidean inner product and $\| \cdot \|$ denotes the induced norm. In a slight abuse of notation, we will also use $\| \cdot \|$ to denote the induced matrix norm, while $\| \cdot\|_F$ denotes the Frobenius norm. $[m]$ denotes the set of positive integers up to and including $m$, i.e., $[m] = \{1,2,\ldots,m \}$. $\{S_k\}_{k \in [K]}$ denotes a collection of sets indexed by $[K]$, i.e., $\{S_k\}_{k\in [K]} = \{S_k: \: k \in [K]\}$. For a collection of sets $\{S_k\}_{k \in [K]}$, $S_{(k)}$ denotes the $k$-th largest set. For a set $A$, $\text{int}(A)$ denotes its interior. $\mathcal{O}(\cdot)$, $\Omega(\cdot)$ are the standard ``big O'' and ``big Omega'', i.e., for two non-negative sequences $\{a_n\}_{n \in \mathbb{N}}$ and  $\{b_n\}_{n \in \mathbb{N}}$, the relation $a_n = \mathcal{O}(b_n)$ ($a_n = \Omega(b_n)$) implies the existence of a global constant $C_1 > 0$ and $n_1 \in \mathbb{N}$, such that $a_n \leq C_1 b_n$ ($a_n \geq C_1b_n$), for all $n \geq n_1$.

\section{Problem formulation and preliminaries}\label{sec:problem}

\noindent Consider $m$ users that participate in a DL system. Each user $i \in [m]$ contains a local dataset $x_{ij} \in \mathcal{X}$, $j  \in [n_i]$, distributed according to an unknown local distribution $\mathcal{D}_i$. Here, $n_i \in \mathbb{N}$ is the number of samples available to user $i$ and $\mathcal{X} \subset \mathbb{R}^{d^\prime}$ is the data\footnote{The space $\mathcal{X}$ is also referred to as \emph{feature} space in the literature.} space. Given a loss function $\ell: \Theta \times \mathcal{X} \mapsto \mathbb{R}$, each user forms the local \emph{empirical loss} $f_i: \Theta \mapsto \mathbb{R}$, given by
\begin{equation}\label{eq:local_loss}
    f_i(\theta) = \frac{1}{n_i}\sum_{j = 1}^{n_i}\ell(\theta;x_{ij}).
\end{equation} The local empirical loss represents an estimate of the \emph{population loss} $F_i: \Theta \mapsto \mathbb{R}$, given by
\begin{equation}\label{eq:population_loss}
    F_i(\theta) = \mathbb{E}_{X_i \sim \mathcal{D}_i}[\ell(\theta;X_i)],
\end{equation} where $X_i$ is the data generating random variable distributed according to $\mathcal{D}_i$. The goal of each user is to train a model that performs well with respect to the population loss~\eqref{eq:population_loss}, using only the local empirical loss~\eqref{eq:local_loss} and leveraging the empirical losses of similar users. To make this notion more precise and formally introduce the problem we want to solve, we now state the assumption used throughout the paper. 

\begin{assumption}\label{asmpt:clusters}
There exist $K$ different data distributions in the system, with $1 < K \leq m$, such that each user samples independent, identically distributed (IID) data from only one of the distributions. Let $\theta_k^\star \coloneqq \argmin_{\theta_k \in \Theta}F_k(\theta_k)$ denote population optimal models of each cluster $k \in [K]$ and let $D = \min_{k \neq l}\| \theta_k^\star - \theta_l^\star \|$. We then assume $D > 0$.
\end{assumption}

\begin{remark}
    Assumption~\ref{asmpt:clusters} defines the problem setup we consider. It is standard in clustered learning~\cite{ieee_ghosh,even2022sample,ghosh2019robust}.
\end{remark}

Assumption~\ref{asmpt:clusters} provides a natural partitioning of the set of users $[m]$, given by $\{C_k\}_{k \in [K]}$, where $C_k$'s are mutually disjoint and their union is the whole set $[m]$. Under Assumption~\ref{asmpt:clusters}, the models offering optimal statistical guarantees can be trained by solving
\begin{equation}\label{eq:emp_clust2}
    \argmin_{\theta_1,\ldots,\theta_K \in \Theta}\sum_{k = 1}^K \frac{|C_k|}{m} g_k(\theta_k),
\end{equation} where $g_k: \Theta \mapsto \mathbb{R}$ is the cluster-wise loss, $g_k(\theta) = \frac{1}{|C_k|}\sum_{i \in C_k}f_i(\theta)$. To see why such an approach is optimal recall that, for a user $i \in C_k$, the optimal population model is given by $\theta_k^\star = \argmin_{\theta_k \in \Theta}F_k(\theta_k)$. If we denote the empirical risk minimizers (ERMs) as $\widehat{\theta}_i = \argmin_{\theta_i \in \Theta}f_i(\theta_i)$, applying results from, e.g.,~\cite{Shalev-Shwartz_sco}, gives $\|\theta^\star_k - \widehat{\theta}_i\|^2 = \mathcal{O}\left(\frac{1}{n_i}\right)$ (with high probability). Let  $\widehat{\theta}_k$ be minimizer of the empirical cluster-wise cost, i.e., $\widehat{\theta}_k = \argmin_{\theta_k \in \Theta}g_k(\theta_k)$. We then have $\|\theta^\star_k - \widehat{\theta}_k\|^2 = \mathcal{O}\left(\frac{1}{n_k}\right)$, where $n_k = \sum_{i \in C_k}n_i$ is the total number of available samples from $\mathcal{D}_k$, showing the benefits of CL. Actually, the $\mathcal{O}\left( \frac{1}{n_k}\right)$ decay is the best achievable in our setting, see, e.g., the Hajek-Le Cam minimax theorem (Theorem 8.11 in \cite{vaart_1998}) for a formal account of the statement; see also~\cite{zhang-duchi}. On the other hand, by Assumption~\ref{asmpt:clusters} and cluster design, the benefits of further merging (and/or modifying) the clusters, in terms of sample size, can potentially be significantly outweighted by the distribution skew between two different clusters (see Section~\ref{subsec:inexact_clust} ahead). We now state the rest of the assumptions used in the paper.

\begin{assumption}\label{asmpt:param_space}
The parameter space $\Theta \subset \mathbb{R}^d$ is a compact convex set, with $\theta^\star_k \in \text{int}(\Theta)$, for all $k \in [K]$. Let $R > 0$ be such that $\|\theta\| \leq R$, for all $\theta \in \Theta$.
\end{assumption}

\begin{remark}
Assumption~\ref{asmpt:param_space} is a standard assumption on the parameter space in statistical learning, e.g.,~\cite{zhang-duchi,Shalev-Shwartz_sco,erm-sco}.
\end{remark}

\begin{assumption}\label{asmpt:loss_function}
For any fixed $x \in \mathcal{X}$ and any $\theta, \theta^\prime \in \Theta$, the loss function $\ell$ is: 1) nonnegative, i.e., $\ell(\theta;x) \geq 0$; 2) convex, i.e., $\ell(\theta^\prime;x) \geq \ell(\theta;x) + \langle \nabla \ell(\theta;x), \theta^\prime - \theta \rangle$; 3) $L$-smooth, i.e., $\ell(\theta^\prime;x) \leq \ell(\theta;x) + \left\langle \nabla \ell(\theta;x), \theta^\prime - \theta \right\rangle + \frac{L}{2}\|\theta^\prime - \theta\|^2$.
\end{assumption}

\begin{remark}\label{rem:smooth_loss}
Note that Assumption~\ref{asmpt:loss_function} implies non-negativity, convexity and  $L$-smoothness of all of $F_k$, $g_k$, $f_i$, $k \in [K]$, $i \in [m]$. Moreover, under convexity of $\ell$, the smoothness condition is equivalent to Lipschitz continuous gradient of $\ell$, i.e., for any $x \in \mathcal{X}$ and any $\theta, \theta^\prime \in \Theta$, $\| \nabla \ell(\theta;x) - \nabla \ell(\theta^\prime;x)\| \leq L\| \theta - \theta^\prime \|$.
\end{remark}

\begin{assumption}\label{asmpt:pop_loss}
For each $k \in [K]$, the population loss $F_k(\theta) = \mathbb{E}_{X_k \sim \mathcal{D}_k}[\ell(\theta;X_k)]$ is $\mu_{F_k}$-strongly convex, i.e., for all $\theta, \theta^\prime \in \Theta$, $F_k(\theta^\prime) \geq F_k(\theta) + \langle \nabla F_k(\theta), \theta^\prime - \theta \rangle + \frac{\mu_{F_k}}{2}\|\theta - \theta^\prime \|^2$. Denote by $\mu_{F} = \min_{k \in [K]}\mu_{F_k}$.
\end{assumption}

\begin{remark}
    Assumption~\ref{asmpt:pop_loss} ensures that the population loss is well-behaved and allows the study of MSE of the models. If Assumption~\ref{asmpt:pop_loss} is relaxed to allow for convex, or non-convex population loss, it would no longer be possible to analyze the MSE rate. Different measures of convergence (in expectation) would have to be considered , such as the optimality gap, i.e., $\mathbb{E}\left[ F(\theta) - F(\theta^\star)\right]$, or the gradient norm-squared, i.e., $\mathbb{E}\|\nabla F(\theta)\|^2$, for convex and non-convex population loss, respectively. Relaxing the strong convexity assumption on the population loss is an interesting direction for future work.
\end{remark}

\begin{assumption}\label{asmpt:loc_behaviour}
For each $k \in [K]$, there exists a neighborhood $U_k = \{\theta \in \Theta: \|\theta - \theta^\star_k\| \leq \rho_k \}$, where $\rho_k > 0$ and constants $H_k > 0$ such that, for all $x \in \mathcal{X}$ and any $\theta, \theta^\prime \in U_k$, the loss $\ell$ satisfies $\| \nabla^2 \ell(\theta;x) - \nabla^2 \ell(\theta^\prime;x)\| \leq H_k\|\theta - \theta^\prime\|$.
\end{assumption}

\begin{remark}
Assumption~\ref{asmpt:loc_behaviour} requires the loss $\ell$ to be well-behaved in neighborhoods of the population optimal models. Akin to Assumption 3 in~\cite{zhang-duchi}, this assumption is necessary for averaging to work. As our family of methods uses model averaging, such an assumption is unavoidable. We refer the reader to~\cite{zhang-duchi} for an elaborate discussion on this requirement.
\end{remark}

\begin{assumption}\label{asmpt:unif_bdd_ptws}
    There exist constants $N_k > 0$, $k \in [K]$, such that, for all $x \in \mathcal{X}$, $\|\nabla_{\theta}\ell(\theta_k^\star,x)\| \leq N_k$.
\end{assumption}

\begin{remark}
    Assumption~\ref{asmpt:unif_bdd_ptws} requires the gradients of $\ell$ to be uniformly bounded, in only a finite number of points (the population optimas $\theta^\star_k$, $k \in [K]$). This is a mild assumption on the behaviour of the gradients, also made in~\cite{erm-sco}.
\end{remark}

\begin{remark}
    Assumption~\ref{asmpt:unif_bdd_ptws} can be compared to the subgaussian noise assumption in~\cite{even2022sample}. While Assumption~\ref{asmpt:unif_bdd_ptws} implies uniformly bounded noise for a finite number of points, it does not make any assumptions on the behaviour of the noise on the rest of the parameter space. On the other hand, while subgaussian noise is less strict than assuming uniform boundedness, it is imposed on the whole parameter space.
\end{remark}

From Remark~\ref{rem:smooth_loss} and compactness of $\Theta$, we can see that all of $F_k$'s have bounded gradients. Denote the bounds by $G_{F_k} = \max_{\theta \in \Theta}\|\nabla F_k(\theta)\|$, $k \in [K]$. Appealing to the mean value theorem, we can see that $F_k$'s are Lipschitz continuous, with constants $G_{F_k}$. From the $L$-Lipschitz continuous gradients of $\ell$ and the continuity of the Hessian of $\ell$ on $U_k$, we can conclude that the Hessian of $\ell$ is uniformly bounded on $U_k, \: k \in [K]$, with the bounding constant being equal to $L$, i.e., $L = \max_{x \in \mathcal{X},\theta \in U_k}\|\nabla^2 \ell(\theta;x)\|$, for all $k \in [K]$. 

Note that the formulation~\eqref{eq:emp_clust2} implicitly assumes the knowledge of the true clustering structure. In reality, the distributions and their associated clustering structures are not known. Moreover, even the exact number of different distributions, $K$, is typically not available. Therefore, the formulation~\eqref{eq:emp_clust2} is impossible to obtain and solve in practice. In what follows, we propose a family of methods that is able to deal with these issues, by correctly identifying the true clusters and producing models that offer the same order-optimal MSE guarantees, as the models with knowledge of the true clustering structure, e.g., ones trained using~\eqref{eq:emp_clust2}.

\section{The ODCL-$\mathcal{C}$ family of methods}\label{sec:method}

\noindent In this section, we outline the family of one-shot methods, ODLC-$\mathcal{C}$. To that end, we first define a general cluster separability condition and the notion of admissible clustering algorithms used in our framework.

Let $\{a_i\}_{i \in [n]} \subset \mathbb{R}^d$ be a set of $n$ datapoints and $\{C_k\}_{k \in [K]}$ be a disjoint clustering of the datapoints into $K$ clusters, with $C_k = \{i \in [n]: a_i \text{ belongs to cluster } k\}$. Let  $\mu_k = \frac{1}{|C_k|}\sum_{i \in C_k}a_i$, $k \in [K]$, denote the cluster means.

\begin{definition}\label{def:sep_cond}
     We say that the dataset $\{a_i\}_{i \in [n]}$ is \emph{separable} with respect to the clustering $\{C_k \}_{k \in [K]}$, if for some $\alpha > 1$ we have
    \begin{equation}\label{eq:recovery_condition}
        \alpha\max_{i \in C_k, k \in [K]}\|\mu_k - a_i\| < \min_{l \neq k}\|\mu_k - \mu_l \|.
    \end{equation} 
\end{definition}

\begin{remark}\label{rmk:alpha}
    Equation~\eqref{eq:recovery_condition} has a simple interpretation. The left-hand side of~\eqref{eq:recovery_condition} is the largest cluster radius, with the right-hand side being the minimal distance between different cluster centers. Therefore,~\eqref{eq:recovery_condition} requires that the clusters be compact and sufficiently separated. From this perspective, the parameter $\alpha$ can be seen as a measure of separation of the clusters, with larger $\alpha$ implying more separated clusters. In Remark~\ref{rmk:dual_alpha} ahead we highlight another interpretation of $\alpha$, as a measure of efficiency of a clustering algorithm.
\end{remark}

In what follows, we will refer to a dataset $\{a_i\}_{i \in [n]}$ as separable if there exists some underlying clustering $\{C_k\}_{k \in [K]}$, such that the condition~\eqref{eq:recovery_condition} is satisfied for some $\alpha > 1$. Next, let $\Xi(\cdot)$ be a clustering algorithm that takes as input some hyperparameter(s) $\eta$. Some examples of clustering algorithms and hyperparameters are presented next.

1) \emph{Convex clustering} - choosing $\Xi$ to be the convex clustering algorithm, the hyperparameter $\eta$ is the convex clustering penalty parameter, $\lambda > 0$. A formal description of the algorithm is provided in Appendix~\ref{app:special_algos}.

2) \emph{$K$-means} - choosing $\Xi$ to be the $K$-means's algorithm, the hyperparameter $\eta$ is the number of clusters, $K \in \mathbb{N}$. A formal description of the algorithm is provided in Appendix~\ref{app:special_algos}.

3) \emph{Gradient clustering} - choosing $\Xi$ to be the gradient clustering algorithm, the hyperparameters $\eta$ are the number of clusters $K \in \mathbb{N}$ and the step-size $\alpha > 0$. For further details, the reader is referred to~\cite{armacki2022gradient}.

From the examples above we can see that different clustering methods exhibit different requirements with respect to apriori knowledge of parameters. For example, both $K$-means and gradient clustering require knowledge of the true number of clusters $K$, while convex clustering requires a specific choice of the penalty parameter $\lambda$ in order to recover the true clustering. Next, we define the class of \emph{admissible} clustering algorithms, used in our framework.  

\begin{definition}
    We say that a clustering algorithm $\Xi$ is \emph{admissible} if for any separable dataset $\{a_i \}_{i \in [n]}$, there exist hyperparameter(s) $\eta$, such that $\Xi(\eta)$ exactly recovers the underlying clustering. Denote the class of admissible clustering algorithms as $\mathcal{C}$.
\end{definition}

In Section~\ref{subsec:cvx_km} we show that both $K$-means and convex clustering belong to the class of admissible algorithms. In order for each user to obtain an approximation of its population loss minimizer in the absence of knowing the true clustering structure, and while pursuing a single communication round in a server-client setting, we propose a method that is summarized in Algorithm~\ref{alg:one-shot}. Initially, each user obtains a model by minimizing the local empirical loss and sends the model to the server. This can be achieved by implementing any standard learning algorithm, e.g., gradient descent. Next, the server receives selects a clustering algorithm and performs clustering on the received models, obtaining clusters of the users' local models. Finally, the server averages the local models and sends the averaged models to the users, according to cluster membership. As such, the algorithm is straightforward to implement in practice. A diagram illustrating the workflow of the proposed method on a toy example is presented in Figure~\ref{fig:algorithm}. Some practical guidelines on selecting hyperparameter(s) $\eta$ for different clustering algorithms can be found in Appendix~\ref{app:practical}.

\begin{algorithm}[tb]
   \caption{ODCL-$\mathcal{C}$}
   \label{alg:one-shot}
\begin{algorithmic}
   \STATE {\bfseries 1. Each user $i \in [m]$:} obtain a local model $\widehat{\theta}_i = \argmin_{\theta_i \in \Theta}f_i(\theta_i)$ and send it to the server.
   \STATE {\bfseries 2. The server:} (i) Receive models $\{\widehat{\theta}_i\}_{i \in [m]}$, choose a clustering algorithm $\Xi \in \mathcal{C}$ and hyperparameter(s) $\eta$.
   \STATE (ii) Run the clustering algorithm $\Xi(\eta)$ on $\{\widehat{\theta}_i \}_{i \in [m]}$, resulting in $K^\prime$ clusters $\{C_k^\prime \}_{k \in [K^\prime]}$.
   \STATE (iii) Average local models according to the resulting clusters, i.e., $\widetilde{\theta}_k = \frac{1}{|C_k^\prime|}\sum_{i \in C_k^\prime}\widehat{\theta}_i$, for all $k \in [K^\prime]$.
   \STATE (iv) Send the models according to cluster assignment, i.e., send $\widetilde{\theta}_k$ to each $i \in C_k^\prime$, for all $k \in [K^\prime]$. 
\end{algorithmic}
\end{algorithm}

\begin{figure}[!t]
    \centering
    \includegraphics[width=0.7\columnwidth]{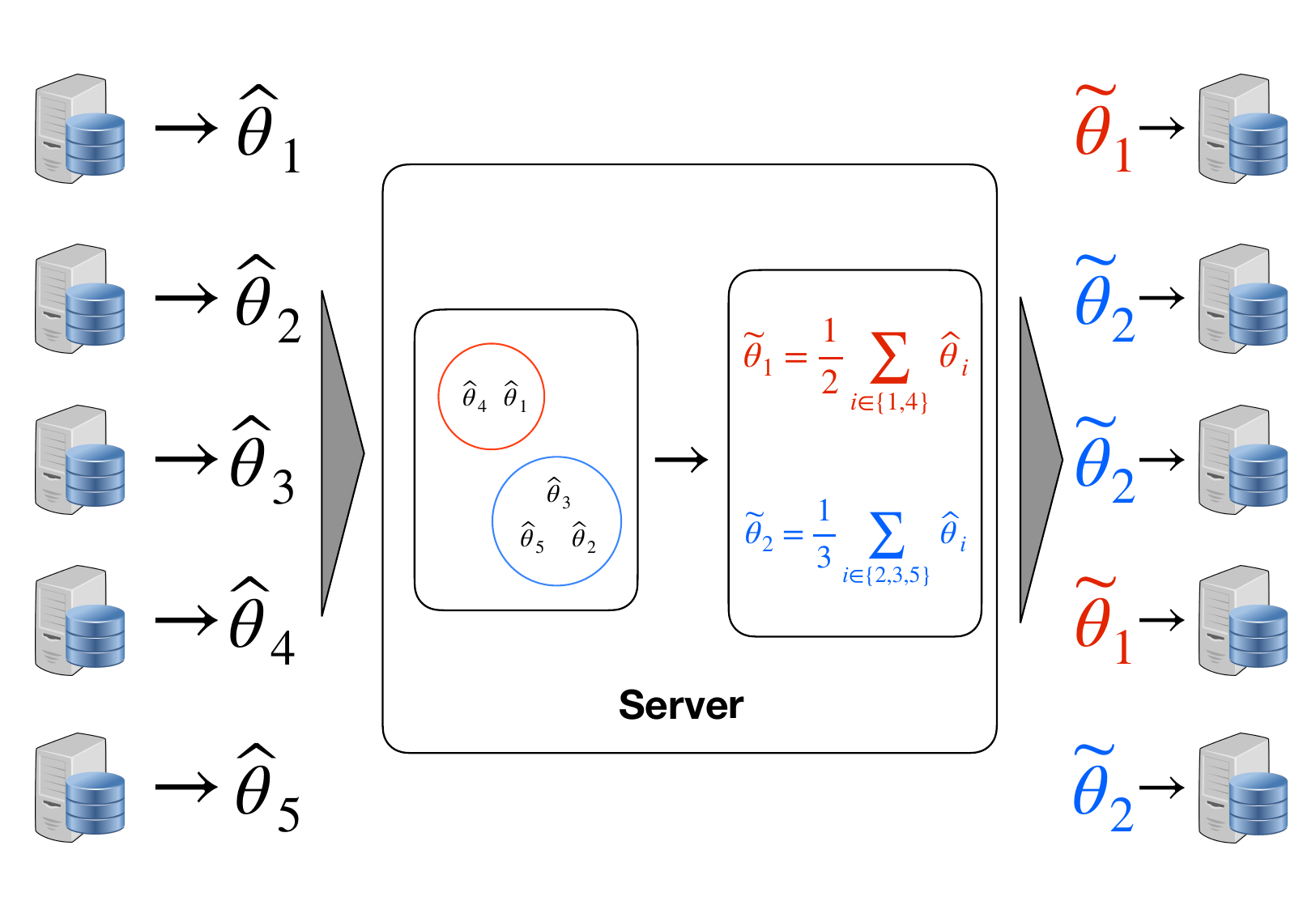}
    \caption{A diagram of the workflow of the ODCL-$\cal C$ method, with $m = 5$ users. Initially, each user produces the local ERM $\widehat{\theta}_i$. Next, users upload their models to the server. The server runs a clustering algorithm, producing two clusters and averages the ERMs according to the clustering. Finally, the resulting models $\widetilde{\theta}_1$ and $\widetilde{\theta}_2$ are broadcast to users, according to cluster membership.}
    \label{fig:algorithm}
\end{figure}

\begin{remark}
     Computation wise, the key bottlenecks of the algorithm are: 1. computing the local ERM; and 2. the clustering step. Local ERMs can be computed by any optimization algorithm and the cost depends on the problem related parameters, e.g., condition number. A method that trades computation costs of finding local ERMs for sample requirements is studied in Theorem~\ref{thm:inexact_erm} ahead. The clustering step can be computationally intense, depending on the problem dimension $d$ and the number of users $m$. The complexity of this step also depends on the choice of the clustering algorithm, e.g., depending on the availability of a good estimate of the number of clusters, one can use $K$-means or convex clustering, with $K$-means incurring a much lower computation cost.
\end{remark}
\begin{remark}
    Compared to traditional FL and one-shot methods with heterogeneous data, our method offers \emph{personalization}, as the former methods produce a \emph{single model}. This one-size-fits-all approach is oblivious to statistical heterogeneity and needs of different users, and often tends to be biased in the direction of the users that perform and send the largest number of updates~\cite{wang2020tackling}, or have the largest number of samples~\cite{yu2020salvaging}, effectively ignoring the users with smaller datasets or less computational power. On the other hand, our method aims to identify the presence of different distributions and harness information from all the users belonging to the distribution. Moreover, our approach treats the contribution of all users equally, by uniformly averaging the models.
\end{remark}

\section{Theoretical guarantees}\label{sec:theory}

\noindent In this section, we present theoretical guarantees of the proposed family ODCL-$\mathcal{C}$. For the sake of simplicity, we assume $n_i = n$, for all $i \in [m]$. Section~\ref{subsec:main} presents the main result. Section~\ref{subsec:cvx_km} specializes the result to ODCL-KM and ODCL-CC. Section~\ref{subsec:comparison} compares our methods with the methods in~\cite{ieee_ghosh,even2022sample}. Section~\ref{subsec:inexact} allows for the local ERMs to be solved inexactly, while Section~\ref{subsec:inexact_clust} presents a method with reduced sample requirements.

\subsection{Main results}\label{subsec:main}

\noindent In this section we present the main result of the paper, establishing that our methods achieve order-optimal MSE rate. The proofs from this section can be found in Appendix~\ref{app:proof_thm1}.

\begin{theorem}\label{thm:main}
Let Assumptions~\ref{asmpt:clusters}-\ref{asmpt:unif_bdd_ptws} hold and let $\mathcal{C}$ be the class of admissible clustering algorithms. If the number of samples per user satisfies $n \geq 3$ and $\frac{n}{\log n} > \frac{16M\alpha^2}{D^2}$, where $M = \max_{i \in C_k, k \in [K]}M_{ik}$, with $M_{ik}$ problem dependent constants defined in Appendix~\ref{app:proof_thm1}, then for all $k \in [K]$, the models produced by ODCL-$\mathcal{C}$ achieve the MSE
\begin{align*}
    \mathbb{E}\|\widetilde{\theta}_k - \theta^\star_k\|^2 &= \mathcal{O}\Bigg(\frac{N_k^2}{n|C_k|\mu_{F_k}^2} + \frac{R^2N^2}{n|C_{(K)}|\mu_F^2D^2}  + \frac{\lambda_k + \nicefrac{\lambda R^2}{D^2}}{n^2} + \frac{\upsilon_k}{n^2|C_k|} + \frac{R^2\upsilon}{n^2|C_{(K)}|D^2} \\ &+ \frac{\nu_k + \nicefrac{\nu R^2}{D^2} + R^2m}{n^3} + \left(dR^2 + \frac{\tau_k}{n} \right)de^{-\alpha_k n} + \left(\frac{dKR^4}{D^2} + \frac{R^2\tau}{nD^2}\right)de^{-\alpha n} \Bigg)
\end{align*} where $N^2$,$\lambda$,$\upsilon$,$\tau$,$\nu$ are the sums of the cluster-wise values, e.g., $\nu = \sum_{k \in [K]}\nu_k$; $\mu_F = \min_{k \in [K]}\mu_{F_k}$, $\alpha = \min_{k \in [K]}\alpha_k$, with $\upsilon_k,\lambda_k,\nu_k,\tau_k$,$\alpha_k$ defined in Lemma~\ref{lm:duchi}, Appendix~\ref{app:proof_thm1}.
\end{theorem}

\noindent \textit{Proof outline}. The high-level idea of the proof is to construct an event $\Psi$ on which the clustering algorithm recovers the true clusters. It then follows by Lemma~\ref{lm:duchi} that on $\Psi$, the models produced by our method have order-optimal MSE. Next, we show that $\Psi$ occurs with probability of the order $1 - \Omega\left(\frac{1}{n^2} \right)$. Combining the two observations and using the decomposition
\begin{align*}
    \mathbb{E}\|\widetilde{\theta}_k - \theta_k^\star \|^2 = \mathbb{E}\mathbb{I}_{\Psi}\|\widetilde{\theta}_k - \theta_k^\star\|^2 + \mathbb{E}\mathbb{I}_{\Psi^c}\|\widetilde{\theta}_k - \theta_k^\star\|^2 \leq \mathbb{E}\|\overline{\theta}_k - \theta_k^\star\|^2 + 4R^2\mathbb{P}(\Psi^c), 
\end{align*} it follows that the MSE of our method is order-optimal. Here, $\overline{\theta}_k = \frac{1}{|C_k|}\sum_{i \in C_k}\widehat{\theta}_i$ is the true average across the users belonging to cluster $C_k$, while $\mathbb{I}_{\Psi}$ is the indicator of $\Psi$. The full proof, detailing the construction of $\Psi$, showing that the correct clustering is recovered on $\Psi$ and that $\Psi$ occurs with sufficiently high probability, can be found in Appendix~\ref{app:proof_thm1}.

Theorem~\ref{thm:main} provides the MSE rate of the proposed family. This rate is equivalent to the one achieved by training a centralized learner on each cluster, while simultaneously achieving significant communication savings, requiring only a single communication round.

\begin{remark}\label{rmk:num_samples}
There are two sets of parameters that affect the performance of our methods, hyperparameters of the clustering algorithm and problem related parameters. Hyperparameters of the clustering algorithm, such as the penalty parameter $\lambda$ for convex clustering, or the number of clusters for $K$-means, can be controlled by the user and directly affect the results, via the quality of produced clusters. The problem related parameters affect the performance indirectly, via the difficulty of the resulting problem, e.g., we can identify three components of the condition on the number of samples in Theorem~\ref{thm:main}, that reflect the complexity of different aspects of the problem: 1) $M$ reflects the difficulty of the learning problem, as it depends on parameters like the condition number; 2) $\alpha$ reflects the efficiency of the clustering algorithm (see Remark \ref{rmk:dual_alpha} ahead); 3) $D$ reflects the difficulty of the clustering problem (recall Assumption~\ref{asmpt:clusters}).
\end{remark}

\begin{remark}
    Theorem 1 suggests that the main limitation of the proposed method in practice is the requirement on the number of available samples. In particular, if users have an insufficient number of samples, the local ERMs may be far from the population models and the clustering algorithm may fail to identify the true clusters. This can potentially produce models that do not perform well. We explore this in numerical simulations in Section~\ref{sec:numerical}, where we consider settings in which users have only a few samples $n$, with dimension $d \gg n$.
\end{remark}

\begin{remark}\label{rmk:rv1discussion}
    Theorem~\ref{thm:main} gives the learning guarantees achievable when only one communication round is possible. Assuming a few-shot communication regime, e.g., $K$-shot, we first note that that our algorithm can not be directly extended to $K$-shot, as it learns the local ERMs from scratch, producing unrelated models in each round. Consider any algorithm ran for $K$ communication rounds. Such a naive implementation can be seen as an intermediary between full communication, i.e., communicating until the problem is solved exactly, and one-shot considered in this work. Assuming, for the sake of simplicity, the data is homogeneous (IID), the problem reduces to learning a single global model. Then, it is well known that both one-shot and full communication methods achieve order-optimal MSE. On the other hand, $K$-shot generalizes sub-optimally, as the training is interrupted after only a few rounds and the resulting model can be very far from the true optimum (see, e.g., Section 5.4 in~\cite{Shalev-Shwartz_sco}). Extending our method to a $K$-shot scenario would require a fundamental change in the algorithm and is an interesting direction for future work.
\end{remark}

\begin{remark}
    The impact of noisy data points, such as outliers, on our method is two-pronged. We require a sufficiently well-behaved loss function, via Assumptions~\ref{asmpt:loss_function} and~\ref{asmpt:unif_bdd_ptws}, which implies the following bound, for any $x \in \mathcal{X}$, $\theta \in \Theta$, and $k \in [K]$, 
    \begin{align*}
        \|\nabla \ell(\theta;x) - \nabla F_k(\theta) \| &= \|\nabla \ell(\theta;x) - \nabla F_k(\theta) \pm \nabla \ell(\theta;x^\star)\| \leq \|\ell(\theta;x) - \ell(\theta^\star_k;x) \| + \|\ell(\theta^\star_k;x) \| + \|\nabla F_k(\theta) \| \\ &\leq L\|\theta - \theta^\star_k\| + N_k + G_{F_k} \leq 2LR + N_k + G_{F_k},
    \end{align*} i.e., the ``variance'' depends on $L,R,N_k$ and $G_{F_k}$. For small $R$, the value of $G_{F_k}$ will also be small (recall the definition of $G_{F_k}$ in Section~\ref{sec:problem}), and the variance will only depend on $N_k$. Looking at the MSE and condition on the number of samples in Theorem~\ref{thm:main}, we can see that $N_k$ appears in the leading term in the MSE, as well as in the sample complexity requirement, via the term $M$. Therefore, more outliers (i.e., large $N_k$) result in both higher sample requirements and worse MSE rates of our method (in the form of larger problem related constants). The implications are similar if the radius $R$ is large, i.e., more outliers and larger variance results in both higher sample requirements and worse MSE rates of our method.
\end{remark}

\subsection{Guarantees of ODCL-CC and ODCL-KM}\label{subsec:cvx_km}

\noindent In this section we show that both convex clustering and $K$-means are guaranteed to recover the true clustering, if condition~\eqref{eq:recovery_condition} is satisfied (for a specific value of $\alpha$). The proofs, as well as description of the algorithms, are in Appendix~\ref{app:special_algos}.

\begin{lemma}\label{lm:conv_clust_rec}
    If condition~\eqref{eq:recovery_condition} is satisfied with $\alpha_{cc} = \frac{4(m - |C_{(K)}|)}{|C_{(K)}|}$, then there exist values of $\lambda > 0$\footnote{There exists a range of values for $\lambda$ that guarantees that the true clusters will be recovered, given by the interval $\left[\frac{\max_{i \in C_k,k\in [K]}\|\mu_{k} - a_i\|}{|C_{(K)}|}, \frac{\min_{k\neq l,k,l\in [K]}\|\mu_k - \mu_l\|}{2(m - |C_{(K)}|)} \right).$} for which convex clustering recovers the true clusters.
\end{lemma}

\begin{lemma}\label{lm:$K$-means_rec}
    If condition~\eqref{eq:recovery_condition} is satisfied with $\alpha_{km} = 2 + \frac{2c\sqrt{m}}{\sqrt{|C_{(K)}|}}$, for some global constant $c > 0$\footnote{The authors in~\cite{kmeans_awasthi} show that $c \geq 100$ suffices, hence we will assume $c \geq 100$.} and the true number of clusters $K$ is known, then spectral $K$-means clustering recovers the true clusters.
\end{lemma}

\begin{remark}\label{rmk:dual_alpha}
    As discussed in Remark~\ref{rmk:alpha}, $\alpha$ can be interpreted as the parameter that shows how well separated the clusters are. The value of $\alpha$ in Definition $1$ implicitly depends on the dataset and the underlying clustering. On the other hand, the critical values of $\alpha$ in Lemmas $1$ and $2$ depend on the specific algorithm used for clustering, i.e., convex clustering and $K$-means, stating the minimal separation (by lower-bounding $\alpha$) for which specific algorithms recover the exact clusters. In that sense, when specializing $\alpha$ for different algorithms, e.g., $\alpha_{km}$ and $\alpha_{cc}$ for $K$-means and convex clustering, respectively, it can be seen as a measure of efficiency of a clustering algorithm. From this perspective, lower value of $\alpha$ associated with an algorithm implies that the algorithm is more efficient and requires less separation to recover the true clustering.
\end{remark}

As a direct consequence of Lemmas~\ref{lm:conv_clust_rec} and~\ref{lm:$K$-means_rec}, we see that the guarantees of Theorem~\ref{thm:main} hold for ODCL-CC and ODCL-KM, with sample requirements determined by the specific value of $\alpha$. In particular, they are given by

1) ODCL-CC: $\frac{n}{\log n} > \frac{256M(m-|C_{(K)}|)^2}{|C_{(K)}|^2D^2}$,

2) ODCL-KM: $\frac{n}{\log n} > \frac{64M(\sqrt{|C_{(K)}|} + c\sqrt{m})^2}{|C_{(K)}|D^2}$.

Comparing the two expressions, we can see that the main difference is in terms $\frac{(m - |C_{(K)}|)^2}{|C_{(K)}|^2}$ and $\frac{(\sqrt{|C_{(K)}|} + \sqrt{m})^2}{|C_{(K)}|^2}$ for convex clustering and spectral $K$-means, respectively. Thus, the sample requirements of ODCL-CC are larger than those of ODCL-KM by roughly a factor of $\frac{m}{|C_{(K)}|}$. This can be seen as the price that the convex clustering algorithm incurs for not requiring knowledge of the true number of clusters $K$.

\subsection{Comparison with order-optimal methods}\label{subsec:comparison}

\noindent In this section we compare the guarantees of ODCL-$\mathcal{C}$ with the guarantees of \emph{IFCA}~\cite{ieee_ghosh} and \emph{ALL-for-ALL}~\cite{even2022sample}. The main comparison points are summarized in Table~\ref{tab:conver}.

\emph{IFCA} is a method that iteratively updates cluster membership and models, requiring knowledge of the true number of clusters $K$, initialization that is sufficiently close to the true population optimal models and sufficiently separated clusters. Assuming $n \geq |C_{(1)}|$, after $T = \frac{8\kappa m}{|C_{(K)}|}\log\left(\frac{2D}{\varepsilon_\delta}\right)$ communication rounds, with probability at least $1 - \delta$, $\|\theta_k^T - \theta_k^\star\| \leq \varepsilon_\delta$, for all $k \in [K]$, where $\varepsilon_\delta = \widetilde{\mathcal{O}}\left(\delta^{-1}(n|C_{(K)}|)^{-\nicefrac{1}{2}} \right)$.

\emph{ALL-for-ALL} is a method based on peer-to-peer communication, requiring knowledge of the underlying cluster structure for the design of the communication matrix and model averaging. The method achieves the lower-bound for oracle complexity in distributed learning, where after $T = \Theta\left(\frac{\kappa K}{\varepsilon\mu m}\log\left(\frac{1}{\varepsilon}\right)\right)$ communication rounds, $\frac{1}{N}\sum_{i \in [m]}\|\theta^T_i - \theta^\star_{(i)}\| = \widetilde{\mathcal{O}}(\varepsilon)$, where $\theta_{(i)}^\star$ refers to the population optimal model of the population of user $i$. 

From Section~\ref{subsec:main}, we can see that ODCL-$\mathcal{C}$ matches the order-optimal convergence guarantees of both IFCA and ALL-for-ALL, while simultaneously reducing the communication cost by a factor of $\mathcal{O}\left(\frac{\kappa m}{|C_{(K)}|}\log\left( \frac{2D}{\varepsilon_{\delta}}\right) \right)$. This is achieved without any initialization or separation requirements and without needing the knowledge of the cluster structure. If knowledge of $K$ and cluster structure is not available, IFCA and ALL-for-ALL can not be used, while our algorithm ODCL-CC can still be employed.

\subsection{Inexact ERMs}\label{subsec:inexact}

\noindent In this section we present guarantees for inexact methods, when the users are allowed to solve their local ERM up to precision $\varepsilon > 0$. The proofs can be found in Appendix~\ref{app:inexact}. We make the following additional assumptions. 

\noindent In this section we present guarantees for inexact methods, when the users are allowed to solve their local ERM up to precision $\varepsilon > 0$. The proofs can be found in Appendix~\ref{app:inexact}. We make the following additional assumptions. 

\begin{assumption}\label{asmpt:loss_str_cvx}
For all $i \in [m]$, the empirical loss $f_i$ is $\mu_{f_i}$-strongly convex, i.e., for all $\theta, \theta^\prime \in \Theta$, $f_i(\theta^\prime) \geq f_i(\theta) + \left\langle \nabla f_i(\theta), \theta^\prime - \theta \right\rangle + \frac{\mu_{f_i}}{2}\|\theta - \theta^\prime\|^2.$ Denote by $\mu_f = \min_{i \in [m]} \mu_{f_i}$.
\end{assumption}

\begin{assumption}\label{asmpt:stoch_grad}
For each $i \in [m]$ and all $\theta \in \Theta$, the stochastic gradient $\widetilde{\nabla}f_i(\theta)$ of $f_i$, evaluated at $\theta$, is unbiased, i.e., $\mathbb{E}\left[\widetilde{\nabla}f_i(\theta)\right] = \nabla f_i(\theta)$. Additionally, the stochastic gradients have bounded variance almost surely, i.e., there exists a constant $\sigma_i > 0$, such that, for all $\theta \in \Theta$, $\|\widetilde{\nabla}f_i(\theta) - \nabla f_i(\theta)\| \leq \sigma_i$, almost surely.
\end{assumption}

Assumption~\ref{asmpt:loss_str_cvx} allows for each user to apply iterative solvers, to obtain parameters $\widetilde{\theta}_i$ that are $\varepsilon$-close to the true minima. A standard choice is the (projected) stochastic gradient descent (SGD) algorithm~\cite{robbins-monro}. SGD follows a simple update rule, given by $\theta^{t+1} = \Pi_{\Theta}\left(\theta^t - \eta^t\widetilde{\nabla}f^t\right)$, and outputs $\theta_T$, where $T$ is the preset number of iterations. Here, $\theta^t$ is the estimate of the parameter of interest at iteration $t$, $\eta^t$ is the step-size, $\widetilde{\nabla}f^t$ is a stochastic gradient of the objective function $f$ evaluated at $\theta^t$, with $\Pi_{\Theta}(\cdot)$ the projection operator onto the set $\Theta$. Next, from Assumptions~\ref{asmpt:param_space},~\ref{asmpt:loss_function},~\ref{asmpt:unif_bdd_ptws} and~\ref{asmpt:stoch_grad}, it follows that, for all $\theta \in \Theta$ and all $i \in C_k$, $k \in [K]$
\begin{align*}
    \|\widetilde{\nabla}f_i(\theta)\| \leq \sigma_i + \|\nabla f_i(\theta) - \nabla f_i(\theta^\star_k) \| + \|\nabla f_i(\theta^\star_k) \| \leq \sigma_i + L\|\theta - \theta^\star_k\| + N_k \leq \sigma_i + 2LR + N_k,
\end{align*} almost surely. Define $\Gamma_i \coloneqq \sigma_i + 2LR + N_k$ and $\Gamma \coloneqq \max_{i \in [m]}\Gamma_i$.

\begin{remark}
Note that Assumption~\ref{asmpt:stoch_grad} implies uniformly bounded noise, which is necessary for showing high-probability guarantees of the last iterate of SGD. This can be relaxed to allow for sub-Gaussian noise, e.g., using the non-uniformly averaged SGD from~\cite{harvey2019simple}. The algorithm performs the same update as the standard SGD described above, with the only difference being in the final estimator. Specifically, the non-uniformly averaged SGD outputs $\theta_T = \sum^T_{t=1}\frac{t}{\nicefrac{T(T+1)}{2}}\theta_t$. Due to the ease of implementation, we present the results of the standard SGD that outputs the last iterate $\theta_T$.
\end{remark}

\begin{theorem}\label{thm:inexact_erm}
Let Assumptions~\ref{asmpt:clusters}-\ref{asmpt:stoch_grad} hold, and let $\mathcal{C}$ be the class of admissible algorithms. If each user runs SGD locally for $T$ iterations to produce $\theta_i^T$, $i \in [m]$, and the number of samples per user $n$ and the number of local iterations $T$ are such that $n \geq 3$, $T \geq \max \left \{ 15, \frac{4\Gamma^2}{\mu^2_{f}\varepsilon} \right\}$ and moreover $\frac{n}{\log n} > 4M\left(\frac{D^2}{4\alpha^2} - 4\varepsilon \xi_F \right)^{-1}$, $\frac{T}{\log\log\left(T \right)} \geq \bigg(\frac{12\Gamma^2}{\mu_{f}^2} + 8\Gamma(121\Gamma + 1)(1 + 3\log n)\bigg)\frac{1}{\varepsilon^2}$, where $\xi_F = \max_{k \in [K]}\frac{G_{F_k}}{\mu_{F_k}}$, then for all $k \in [K]$, the models produced by inexact ODCL-$\mathcal{C}$ achieve the MSE
\begin{align*}
    \mathbb{E}\|\widetilde{\theta}_k &- \theta_k^\star\|^2 = \mathcal{O}\bigg(\frac{1}{n|C_{(K)}|}  + \frac{1}{n^2} + \frac{1}{n^2|C_{(K)}|} + \frac{m}{n^3} \\ &+ \left(dK + \frac{1}{n}\right)de^{-\alpha n} + \left(\frac{KR^2}{D^2} + 1\right)\varepsilon\bigg).
\end{align*}
\end{theorem}

\begin{remark}
Note that sample complexity implicitly places a requirement on the precision up to which we solve the local ERMs, i.e., $\varepsilon < \frac{D^2}{16\xi_F\alpha^2}.$ This requirement can again be seen in terms of the difficulty of different system aspects. E.g., if the clusters are well separated, i.e., large $D$, we can solve the local ERMs up to moderate, or low precision. Similarly, if the clustering algorithm is more efficient, i.e., small $\alpha$, the local ERMs can be solved to lower precision. Finally, recall that $\xi_F = \max_{k \in K}\nicefrac{G_{F_k}}{\mu_{F_k}}$, where $G_{F_k}$ and $\mu_{F_k}$ are the Lipschitz continuity and strong convexity constants of $F_k$'s, respectively. Hence, for well behaved $F_k$'s, i.e., small $\tau_{F_k}$, the precision to which the local ERMs are solved is relaxed.
\end{remark}

\begin{remark}
Comparing the MSE rates of Theorem~\ref{thm:main} and Theorem~\ref{thm:inexact_erm}, we can see that the main difference is the presence of an additional term in Theorem~\ref{thm:inexact_erm}, that being $\varepsilon\left(1 + R^2KD^2 \right)$, with $\varepsilon > 0$ representing the accuracy up to which we solve the local ERM. We can therefore see that, as long as the local ERMs are solved up to precision $\varepsilon < \min\left\{ \frac{1}{n|C_{(1)}|},\frac{D^2}{16\xi_F\alpha^2} \right\},$ the rates of Theorem~\ref{thm:main} are recovered, i.e., the final MSE is of the order $\mathcal{O}\left(\frac{1}{n|C_{(K)}|}\right)$, for all $k \in [K]$. This in turns leads to a local iteration requirement of $T \geq \max\left\{15,\frac{4\Gamma^2}{\mu_{f}^2\varepsilon}\right\}$ and $\frac{T}{\log\log\left(T \right)} \geq \bigg(\frac{6L\Gamma^2}{\mu_{\ell}^2} + 4L\Gamma(121\Gamma + 1)(1 + 3\log n)\bigg)\frac{1}{\varepsilon^2}$.
\end{remark}

Similarly to Section~\ref{subsec:cvx_km}, we can specialize the results of Theorem~\ref{thm:inexact_erm} when the clustering algorithm of choice is $K$-means and/or convex clustering, to obtain algorithm specific sample complexities. The results are summarized in Table~\ref{tab:conver}.

\subsection{Trading accuracy for sample requirements}\label{subsec:inexact_clust}

\noindent In this section we show how the sample requirements of Theorem~\ref{thm:main} can be relaxed, at the expense of the MSE converging up to an error floor. We explicitly characterize the error floor in terms of clustering error and heterogeneity. The proofs can be found in Appendix~\ref{app:inexact_clust}.

Let $\{a_i \}_{i \in [n]}$ and $\{C_k \}_{k \in [K]}$ be the dataset and underlying clustering, with cluster means denoted by $\{\mu_k \}_{k \in [K]}$. Let $A \in \mathbb{R}^{n \times d}$ be a matrix whose $i$-th row is $a_i$, with $C \in \mathbb{R}^{n \times d}$ the matrix whose $i$-th row is $\mu_k$ if and only if $i \in C_k$. We have the following definition, known as \emph{center separation}~\cite{kmeans_awasthi}. 

\begin{definition}\label{def:cent_sep}
    We say that the dataset $\{a_i \}_{i \in [n]}$ satisfies the center separation condition, if, for some $c > 0$\footnote{The constant $c$ is the same as in the previous section.} sufficiently large, and any $k, l \in [K]$, $\|\mu_k - \mu_l\| \geq c\left(\Delta_k + \Delta_l\right),$ where $\Delta_k = \frac{1}{\sqrt{|C_k|}}\min\left\{\sqrt{K}\|A - C\|, \|A - C\|_F \right\}$, $k \in [K]$.
\end{definition}

Using simple algebraic manipulations, one can show that Definition~\ref{def:cent_sep} is a relaxation of Definition~\ref{def:sep_cond}, in the sense that Definition~\ref{def:sep_cond} implies Definition~\ref{def:cent_sep}, for the choice $\alpha = \frac{2c\sqrt{m}}{|C_{(K)}|}$. Running only part I of the spectral $K$-means algorithm (see Appendix~\ref{app:special_algos}), we have the following result. 

\begin{theorem}\label{thm:inexact_clust}
    Let Assumptions~\ref{asmpt:clusters}-\ref{asmpt:unif_bdd_ptws} hold. If we run part I of the spectral $K$-means algorithm and the number of samples per user satisfies $\frac{n}{\log n} > \frac{64c^2m\overline{M}}{D^2|C_{(K)}|}$, where $\overline{M} = \frac{1}{m}\sum_{i \in C_k,k \in [K]}M_{ik}$, with $M_{ik}$ problem dependent constants defined in Appendix~\ref{app:inexact_clust}, then the clustering $\{C_k^\prime \}_{k \in [K]}$ produced by the algorithm has the following properties, for any $k \in [K]$
    
    1) $|C_k^\prime \cap C_k| \geq \left(1 - \frac{128}{c^2}\right)|C_k|$.
    
    2) $\sum_{l \neq k}|C_k^\prime \cap C_l| \leq \frac{128}{c^2}|C_k|$.
    
    3) The cluster averaged models achieve the following MSE
    \begin{equation*}
        \mathbb{E}\|\widetilde{\theta}_k - \theta^\star_k \|^2 = \mathcal{O}\left(\frac{1}{n|C_{(K)}^\prime|} + \frac{1}{n^2} + \sum_{l \neq k}\epsilon_{kl}^2D_{kl} \right),
    \end{equation*} where $\epsilon_{kl} = \frac{|C_k^\prime \cap C_l|}{|C_k^\prime|}$ represents the fraction of points from cluster $l$ that were wrongly clustered in cluster $k$, with $D_{kl} = \|\theta_k^\star - \theta_l^\star \|^2$ the pairwise squared difference of the population optima. Moreover, we have that $\epsilon_k = \sum_{l \neq k}\epsilon_{kl} \leq \frac{128}{c^2 - 128}$.
\end{theorem}

\begin{remark}
    The authors in~\cite{kmeans_awasthi} state that $c$ should be treated as a large constant, e.g., $c \geq 100$. Specializing results from Theorem~\ref{thm:inexact_clust}, we then have that, in the worst case, at least $98.7\%$ of points will be clustered correctly, while at most $1.3\%$ of points will be added from other clusters, for every cluster $k \in [K]$. Looking at the non-vanishing term in the MSE rate, we can see that $\sum_{l \neq k}\epsilon_{kl}^2D_{kl} \leq D_{\max}\left(\sum_{l \neq k}\epsilon_{kl}\right)^2 \leq 1.69\times10^{-4}D_{\max}$, where $D_{\max} = \max_{k \neq l}D_{kl}$. Therefore, then the non-vanishing error is at most $\approx 10^{-4}D_{\max}$.
\end{remark}

\begin{remark}
    The procedure in Theorem~\ref{thm:inexact_clust} only uses a part of the spectral $K$-means algorithm, making it both faster and computationally cheaper. Comparing the sample complexity in Theorem~\ref{thm:inexact_clust}, $\Omega\left(\frac{c^2m\overline{M}}{D^2|C_{(K)}|} \right)$, to the one of ODCL-KM that guarantees exact cluster recovery, $\Omega\left(\frac{c^2M(\sqrt{m} + \sqrt{|C_{(K)}|})^2}{D^2|C_{(K)}|} \right)$, we can note that the sample complexity is reduced by an additive factor of the order $\Omega\left(\frac{c^2M}{D^2} \right)$. Moreover, the constant $M = \max_{i \in C_k, k \in [K]}M_{ik}$ is replaced by $\overline{M} = \frac{1}{m}\sum_{i \in C_k, k \in [K]}M_{ik}$, i.e., the worst-case constant is replaced with the average, which can be significant, e.g., in scenarios where only a few users have poorly conditioned problems.
\end{remark}

\begin{remark}
    Combining results from Theorems~\ref{thm:main},~\ref{thm:inexact_clust}, we get the following: if the number of samples per user satisfies $\frac{n}{\log n} > \Omega\left(\frac{c^2Mm}{D^2|C_{(K)}|} + \frac{c^2M}{D^2}\right)$, an order-optimal, asymptotically vanishing MSE rate is guaranteed by the proposed methods. On the other hand, if the number of samples per user satisfies a lower sample requirement of $\Omega\left(\frac{c^2\overline{M}m}{D^2|C_{(K)}|}\right) < \frac{n}{\log n} < \mathcal{O}\left(\frac{c^2Mm}{D^2|C_{(K)}|} + \frac{c^2M}{D^2}\right)$, a order-optimal MSE is guaranteed, with a non-vanishing error term, of the worst-case order $\mathcal{O}\left(10^{-4}D_{\max} \right)$. 
\end{remark}

\section{Numerical experiments}\label{sec:numerical}

In this section we present numerical experiments. All experiments were implemented in python and averaged across 10 runs. To solve the local empirical risk problems and perform convex clustering, we use CVXPY~\cite{diamond2016cvxpy}. Additional extensive experiments comparing the performance of the proposed methods and IFCA, experiments on logistic losses, and more can be found in Appendix~\ref{app:numerical}.

We present experiments on synthetic and real data. For the synthetic data experiments, we consider a linear regression problem with quadratic loss. The data generating process for each cluster is given by $y = \langle x, u_k^\star \rangle + \epsilon$, where $\epsilon \sim \mathcal{N}(0,1)$ follows a standard Gaussian distribution. The number of clusters is set to $K = 10$. The problem dimension is $d = 20$, with each component of the cluster optimal models $u_k^\star$ drawn independently from a uniform distribution. Specifically, $u_{1i}^\star \sim \text{U}([1,2])$, $u_{2i}^\star \sim \text{U}([4,5])$, $u_{3i}^\star \sim \text{U}([7,8])$, $u_{4i}^\star \sim \text{U}([10,11])$ and $u_{5i}^\star \sim \text{U}([13,14])$, with $u_6^\star$ through $u_{10}^\star$ being generated from the corresponding negative intervals, i.e., $u_{6i}^\star \sim \text{U}([-2,-1])$, through to $u_{10i}^\star \sim \text{U}([-14,-13])$, respectively, for all $i \in [d]$. This way, we ensure that $D > 0$. For each datapoint $x \in \mathbb{R}^d$, we chose $5$ random components in $[d]$ that are distributed according to $\mathcal{N}(0,1)$, while the other components are set to zero. For real data experiments, we use the MNIST~\cite{lecun-mnisthandwrittendigit-2010} and CIFAR10~\cite{krizhevsky2009learning} datasets and consider a binary classification logistic regression problem. The number of clusters is set to $K = 2$. Each user contains elements from the first two classes (i.e., the digits $1$ and $2$ for MNIST and the classes \emph{airplane} and \emph{automobile} for CIFAR10), where users from one cluster assign the labels $\{+1,-1\}$ to classes $1$ and $2$, respectively, while the users from the other cluster assign the opposite labels, i.e., $\{-1,+1\}$ to classes $1$ and $2$, respectively. This scenario can be seen as modeling the opposite preference, e.g., when different groups place different values on the same items. Such is the case in many applications, e.g., the next word prediction problem trained on teenager and elderly people's data~\cite{Sattler_clustered}, where the two groups are likely to have different preferences and frequency of use, with respect to the same words. To test the robustness of our method, we perform a third experiment on a noisy version of the above described CIFAR10 data. To create a noisy version of the dataset, for each user we flip $20\%$ of the labels, while the labels in the test data remain unchanged. This models the scenarios where, e.g., erroneous sensors and measurements devices result in erroneous data, when the labeling process is subject to (human or other) errors, or when a malicious attacker intentionally pollutes the data. Flipping the labels induces uncertainty and further heterogeneity within the clusters, making decision making and learning a within-cluster model more difficult, where the flipped points may be considered outliers. 

We consider a DL system with $m = 100$ users and a balanced clustering, i.e., $|C_k| = \frac{m}{K}$. Each user $i \in C_k$ is assigned $n$ IID data points. For the synthetic data experiments we vary the number of samples assigned to each user and evaluate the performance of ODCL-CC and a practical variant of ODCL-KM using $K$-means++ initialization~\cite{kmeans++}, dubbed ODCL-KM++. For the real data experiments, each user is assigned $n = 4$ datapoints for MNIST and $n = 40$ datapoints for CIFAR10, 20 datapoint per class. As the dimensions of the MNIST and CIFAR10 data are $d = 784$ and $d = 3072$ respectively, we test the performance of methods in the scenario when the number of samples per user is much smaller than the problem dimension. Due to a small sample size per user, we evaluate the method with lower sample requirements, ODCL-KM++. 

For synthetic data, we benchmark our methods with: 1) \emph{Oracle Averaging} - oracle that averages models within true clusters; 2) \emph{Cluster Oracle} - oracle that trains models by solving~\eqref{eq:emp_clust2}; 3) \emph{Local ERMs} - ERMs trained on users' local data; 4) \emph{Naive averaging} - averaging models from the whole system. For real data, besides Cluster Oracle and Local ERM, we use IFCA with gradient averaging, with the step-size $0.1$, ran for $T = 200$ rounds and $\tau = 10$ local update steps. To initialize IFCA, we use three approaches: 1) \emph{IFCA-1} - initialize with models trained by Cluster Oracle and add $\mathcal{N}(0,1)$ noise; 2) \emph{IFCA-2} - initialize with models trained by Cluster Oracle and add $\mathcal{N}(0,4)$ noise; 3) \emph{IFCA-R} - random initialization. This way, we study the effect of initialization on the performance of IFCA. Note that in the experiment on noisy CIFAR10 data we do not evaluate the Cluster Oracle method, as the distribution of the training data ($20\%$ of labels flipped) in this experiment does not match that of the test data (no labels flipped). Therefore, in this setting, the Cluster Oracle method is not a true oracle method, and its performance does not provide much information.

To measure the performance in synthetic experiments, we use the average normalized MSE, i.e., for each of the methods, we compute $\frac{1}{m}\sum_{i = 1}^m\frac{\|\widetilde{u}_i - u^\star_{(i)}\|^2}{\|u^\star_{(i)}\|^2}$, where $u_{(i)}^\star$ denotes the population optima associated with user $i$, while $\widetilde{u}_i$ is the model assigned to user $i$. To measure the performance in real data experiments, we evaluate the average test accuracy on a withheld dataset. The results on synthetic data are presented in Figure~\ref{fig:MSE}, with the results on real data presented in Table~\ref{tab:numerical}.

\begin{table}
\caption{Average testing accuracy on MNIST, CIFAR10 and noisy CIFAR10 (CIFAR10N) data.}
\label{tab:numerical}
\begin{center}
\begin{small}
\begin{sc}
\begin{tabular}{lcccr}
\toprule
Method & MNIST & CIFAR10 & CIFAR10N \\
\midrule
Cluster Oracle    & $0.98 \pm 0.003$ & $0.76 \pm 0.006$ & $*$ \\
Local ERM    & $0.83 \pm 0.06$ & $0.70 \pm 0.006$ & $0.62 \pm 0.012$ \\
IFCA-1    & $0.94 \pm 0.03$ & $0.70 \pm 0.12$ & $0.58 \pm 0.07$ \\
IFCA-2    & $0.83 \pm 0.13$ & $0.66 \pm 0.12$ & $0.54 \pm 0.05$ \\
IFCA-R    & $0.65 \pm 0.06$ & $0.51 \pm 0.01$  & $0.52 \pm 0.012$ \\
ODCL-KM++   & $0.91 \pm 0.007$ & $0.76 \pm 0.008$ & $0.66 \pm 0.06$ \\
\bottomrule
\end{tabular}
\end{sc}
\end{small}
\end{center}
\vskip -0.1in
\end{table}

On $y$-axis of Figure~\ref{fig:MSE} we plot the average normalized MSE, with the $x$-axis showing the number of samples per user. Both ODCL-KM and ODCL-CC match the order-optimal performance of the oracle methods. ODCL-KM++ matches the performance of Oracle Averaging for any number of samples, while ODCL-CC requires $n\approx 400$ samples to match the oracle methods. As discussed in Section~\ref{subsec:cvx_km}, this gap stems from the fact that convex clustering needs to estimate both the number of clusters, as well as the clustering itself.

Table~\ref{tab:numerical} presents performance on real data. On MNIST data, ODCL-KM++ significantly improves the accuracy over local models and is close to the oracle method. On the other hand, IFCA is strongly affected by initialization, as its performance varies from close to oracle (IFCA-1) to worse than Local ERM (IFCA-R). Similar observations hold on the CIFAR10 data, where the accuracy of our method is on par with the oracle, while IFCA performs as good as the Local ERM (IFCA-1) at best, and essentially as a random estimator (IFCA-R) at worst. Finally, in the third column, we see that our method performs the best in the presence of noisy data, while all instances of IFCA performe worse than Local ERM.

To study the significance of communication savings of our method, irrespective of heterogeneity, we consider an experiment with linear regression models. We study both the homogeneous, i.e., $K = 1$, and heterogeneous setting, with $K = 4$ clusters and dimension $d = 20$. We benchmark ODCL-KM++ with IFCA. The experiments are performed as follows: fix the number of samples available to each user to $n = 400$. Run ODCL-KM++ and evaluate the averaged MSE across all users. Next, run IFCA until it either matches the MSE error of ODCL-KM++, or performs $T = 1000$ communication rounds. We use IFCA with model averaging, performing $10$ local computation rounds per a communication round, with step-size $0.01$. Note that in homogeneous setting our method evaluates to one-shot averaging, e.g.,~\cite{zhang-duchi}, while IFCA evaluates to FedAvg~\cite{pmlr-v54-mcmahan17a}. In the homogeneous settings, where IFCA is equivalent to FedAvg, we use the zero vector initialize it. In the heterogeneous settings, we initialize IFCA sufficiently close to the population optima, by ensuring that the distance of the initial model from the true population optima for each cluster is at least $\frac{D}{6}$, and at most $\frac{D}{4}$. The results are presented in Table~\ref{tab:table2}. Our method significantly outperforms IFCA in both the homogeneous and heterogeneous setup, reducing the number of communication rounds required for reaching the same MSE by a factor of $400$ and $600$, respectively. 

\begin{table}
    \caption{Number of communication rounds required to achieve the same MSE, with a fixed number of samples.}
    \label{tab:table2}
    \begin{center}
    \begin{small}
    \begin{sc}
    \begin{tabular}{lcccr}
    \toprule
    Method & Homogeneous & Heterogeneous &\\
    \midrule
    ODCL-KM++   & $1$ & $1$ &  \\
    IFCA    & $440$ & $655$ &  \\
\bottomrule
\end{tabular}
\end{sc}
\end{small}
\end{center}
\vskip -0.1in
\end{table}

\begin{figure}[!t]
    \centering
    \includegraphics[width=0.7\columnwidth]{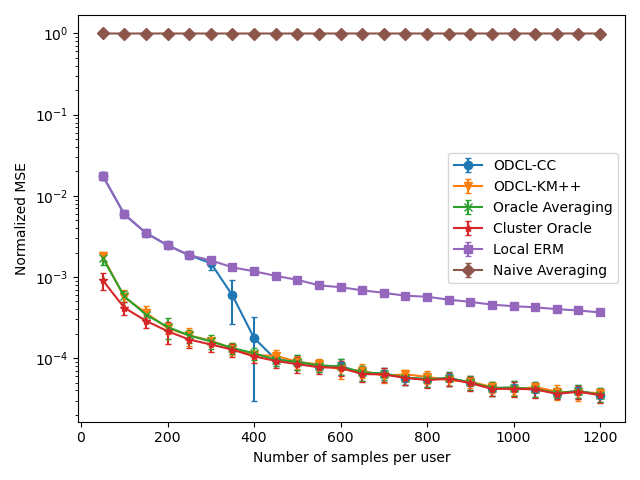}
    \caption{Normalized MSE versus the number of samples available per user. }
    \label{fig:MSE}
\end{figure}

\section{Conclusion}\label{sec:conclusion}

\noindent We proposed a family of one-shot methods for DL, ODCL-$\mathcal{C}$, based on a simple inference and averaging scheme. The resulting methods are communication efficient, requiring a single round of communication. Moreover, the family of methods offers significant flexibility in the choice of algorithms and trade-offs in terms of hyperparameters and resulting sample requirements, as can be seen when specialized to ODCL-KM and ODCL-CC. Our theoretical analysis shows that the methods provide order-optimal MSE rates in terms of sample size, while not requiring any strong separation or initialization conditions, unlike other state-of-the-art algorithms. Numerical experiments on real and synthetic data corroborate our findings and underline the superior performance of our methods. Ideas for future work include extending the framework to the setting where multiple communication rounds are possible, as well as studying the behaviour of the ODCL family under relaxed assumptions, e.g., for general convex and non-convex functions.

\bibliographystyle{unsrt}
\bibliography{bibliography}

\newpage
\appendix
\onecolumn
\section{Outline}

In this Appendix we present the content omitted from the main body of the paper. Section~\ref{app:main} contains the proofs of the main results. Section~\ref{app:practical} contains some practical considerations on selecting hyperparameters for specific  clustering algorithms. Section~\ref{app:numerical} provides additional numerical experiments. Section~\ref{app:lemma} contains the proof of Lemma~\ref{lm:duchi}.

\section{Main results}\label{app:main}

In this section we present the proofs of main results, omitted from the main body of the paper. Subsection~\ref{app:proof_thm1} contains the proof of Theorem~\ref{thm:main}, Subsection~\ref{app:special_algos} contains proofs of Lemmas~\ref{lm:conv_clust_rec},~\ref{lm:$K$-means_rec}, while Subsections~\ref{app:inexact} and~\ref{app:inexact_clust} contain the proofs of Theorems~\ref{thm:inexact_erm} and~\ref{thm:inexact_clust}, respectively.

\subsection{Proof of Theorem~\ref{thm:main}}\label{app:proof_thm1}

\noindent We start by introducing two results used in the proof. 

\begin{lemma}[Theorem 3 in~\cite{erm-sco}]\label{lm:erm}
Under Assumptions~\ref{asmpt:clusters}-\ref{asmpt:pop_loss},~\ref{asmpt:unif_bdd_ptws}, for any $k \in [K]$, any $i \in C_k$ and any $0 < \delta < \frac{1}{2}$, $\epsilon > 0$, with probability at least $1 - 2\delta$, we have
\begin{align*}
    F_k(\widehat{\theta}_i) &- F_k(\theta^\star_k) \leq \frac{16R^2LC(\epsilon,\delta)}{n} + \frac{8RN_k\log \frac{2}{\delta}}{n} \\ &+ \frac{8LF_k(\theta^\star_k)\log \frac{2}{\delta}}{\mu_{F_k} n} + \left(8RL + G_{F_k} + \frac{4RLC(\epsilon,\delta)}{n} \right)\epsilon,
\end{align*} where $C(\epsilon,\delta) \coloneqq 2\left(\log\frac{2}{\delta} + d\log\frac{6R}{\epsilon}\right)$.
\end{lemma}

\begin{lemma}\label{lm:duchi}
Under Assumptions~\ref{asmpt:clusters}-\ref{asmpt:unif_bdd_ptws}, for each $k \in [K]$ and $\overline{\theta}_{k} = \frac{1}{|C_k|}\sum_{i \in C_k}\widehat{\theta}_i$, we have
\begin{align*}
    \mathbb{E}\|\overline{\theta}_{k} - \theta^\star_k \|^2 = \mathcal{O}\Bigg(\frac{N_k^2}{n|C_k|\mu_{F_k}^2} + \frac{\lambda_k + \frac{\upsilon_k}{|C_k|}}{n^2} + \frac{\nu_k}{n^3} + \pi_{k,n}e^{-\alpha_k n}\Bigg),
\end{align*} where $\lambda_k = \frac{N_k^2L^2\log^2 d + \beta_k^{\nicefrac{1}{2}}H_k^2N_k^2}{\mu_{F_k}^4}$, $\upsilon_k = \frac{\beta_k^{\nicefrac{1}{2}}L^2\log^2 d + H_k^2\beta_k}{\mu_{F_k}^2}$, $\nu_k = \beta_k^{\nicefrac{1}{2}}L^4\log^4 d + \beta_k^{\nicefrac{3}{2}}H_k^4$, $\pi_{k,n} = R^2d^2 + \frac{d\tau_k}{n}$, $\alpha_k = \min\left\{\frac{(1 - \rho_k)^2\mu_{F_k}^2\delta_{\rho_k}^2}{8c_1N_k^2}, \frac{\rho_k^2\mu_k^2}{8c_2L^2}\right\}$, $\beta_k = \nicefrac{N_k^4}{(1 - \rho_k)^4\mu_{F_k}^4} + \nicefrac{dR^4}{\alpha_k}$, $\tau_k = \frac{N_k^2}{\mu_{F_k}^2} + \frac{\sqrt{N_k^4L^4\log^4d + \beta_k N_k^4H_k^4}}{\mu_{F_k}^4} + \frac{\sqrt{\beta_k L^8\log^8d + \beta_k^3H_k^8}}{\mu_{F_k}^4}$, with $\delta_{\rho_k} =  \min\{\rho_k, \frac{\rho_k\mu_{F_k}}{2H_k} \}$ and $c_1,c_2 > 0$ global constants. 
\end{lemma}

\begin{remark}
    The proof of Lemma~\ref{lm:duchi} is based on considering an event on which the empirical gradient evaluated at the population optima, as well as the difference between the empirical and population Hessian at the population optima, are small. Leveraging the concentration of empirical gradients and Hessians around their population counterparts (Assumptions~\ref{asmpt:loss_function},~\ref{asmpt:loc_behaviour} and~\ref{asmpt:unif_bdd_ptws}), we then show that the MSE conditioned on this event is of the order $\mathcal{O}\left(\frac{1}{n|C_k|} + \frac{1}{n^2} \right)$. Next, we show that the probability of the event occurring is of the order $1 - \Omega\left(\frac{1}{n^2}\right)$. Combining these observations leads to the desired result. We provide the full proof in Section~\ref{app:lemma} below.
\end{remark}

\begin{proof}[Proof of Theorem~\ref{thm:main}]
We start by noting that, for any event $\Psi$
\begin{equation}\label{eq:setup}
    \mathbb{E}\|\widetilde{\theta}_k - \theta^\star_k \|^2 = \mathbb{E}\|\widetilde{\theta}_k - \theta^\star_k \|^2\mathbb{I}_{\Psi} + \mathbb{E}\|\widetilde{\theta}_k - \theta^\star_k \|^2\mathbb{I}_{\Psi^\mathsf{c}},
\end{equation} where $\mathbb{I}_{\Psi}$ is the indicator random variable, with $\Psi^\mathsf{c}$ denoting the complement of $\Psi$. We proceed to define a specific event $\Psi$ and establish the resulting bounds. Specializing~\eqref{eq:recovery_condition} to our method, we can see that, for the clustering in step (ii) of Algorithm~\ref{alg:one-shot} to be correct, we need the following to hold
\begin{equation}\label{eq:rec_cond}
    \alpha\max_{i \in C_k, k\in [K]}\|\widehat{\theta}_i - \overline{\theta}_k\| \leq \min_{k\neq l, k,l\in[K]}\|\overline{\theta}_k - \overline{\theta}_l \|,     
\end{equation} for some $\alpha > 0$, where $\overline{\theta}_k = \frac{1}{|C_k|}\sum_{i \in C_k}\widehat{\theta}_i$, denotes the models averaged across the true clusters. We will construct $\Psi$ on which the condition~\eqref{eq:rec_cond} is satisfied with high probability. To begin with, note that, for any $i \in C_k$, $k \in [K]$, we have
\begin{equation}\label{eq:bdd_LHS}
\begin{aligned}
    \|\widehat{\theta}_i - \overline{\theta}_k \| \leq \|\widehat{\theta}_i - \theta^\star_k \| + \|\overline{\theta}_k - \theta^\star_k \| \leq \|\widehat{\theta}_i - \theta^\star_k \| + \frac{1}{|C_k|}\sum_{j \in C_k}\|\widehat{\theta}_j - \theta^\star_k \| \leq 2\max_{j \in C_k}\|\widehat{\theta}_j - \theta^\star_k \|.
\end{aligned}
\end{equation} Next, combining Lemma~\ref{lm:erm} with strong convexity of $F_k$ and choosing $\delta = \epsilon = \frac{1}{n^3}$, we get, for all $i \in C_k$, $k \in [K]$
\begin{align}\label{eq:prob_bdd}
    \begin{aligned}
        &\|\widehat{\theta}_i - \theta^\star_k\|^2 \leq \frac{32R^2LC(\epsilon,\delta)}{n\mu_{F_k}} + \frac{16LF_k(\theta^\star_k)(\log 2 + 3 \log n)}{n\mu_{F_k}^2} \\ &+ \frac{16RN_k(\log 2 + 3\log n)}{n\mu_{F_k}} + \frac{\left(16RL + 2G_{F_k} + \frac{8RLC(\epsilon,\delta)}{n} \right)}{n^3\mu_{F_k}},
    \end{aligned}
\end{align} with probability at least $1 - \frac{2}{n^3}$, where $C(\epsilon,\delta) = 2\left(\log 2 + d\log 6R + (d + 1)3\log n \right)$. In terms of the number of samples, $\mathcal{O}\left(\frac{\log n}{n}\right)$ dominates in the above expression. We can therefore upper-bound the right-hand side of~\eqref{eq:prob_bdd} by $\frac{M_{ik}\log n}{n}$, where $M_{ik}$, $i \in C_k$, $k \in [K]$ is defined as follows
\begin{align*}
    M_{ik} &= \frac{16LF_k(\theta^\star_k)(\log 2 + 3)}{\mu_{F_k}^2} + \frac{64R^2L\left(3 + \log 2 + d(\log 6R + 3)\right)}{\mu_{F_k}} \\ &+ \frac{16RN_k(\log 2 + 3)}{\mu_{F_k}}  + \frac{2G_{F_k} + 16RL\left(4 + \log 2 + d(\log 6R + 3)\right)}{\mu_{F_k}}.
\end{align*} Let $M_k = \max_{i \in C_k}M_{ik}$. As $\frac{M_k\log n}{n}$ is an upper bound on the right-hand side of~\eqref{eq:prob_bdd}, we get, for any $i \in C_k$, $k \in [K]$
\begin{equation}\label{eq:prob_bdd2}
    \mathbb{P}\left(\|\widehat{\theta}_i - \theta^\star_k\|^2 \leq \frac{M_k\log n}{n}\right) \geq 1 - \frac{2}{n^3}.
\end{equation} Next, for all $i \in C_k$, $k \in [K]$, define the events
\begin{align*}
    \Sigma_{k} &= \bigg\{\omega: \max_{i \in C_k}\|\widehat{\theta}_i - \overline{\theta}_k\|^2 \leq \frac{4M_k\log n}{n} \bigg\}, \\
    \Upsilon_i &= \left\{\omega: \| \widehat{\theta}_i - \theta^\star_k \|^2 \leq \frac{M_{k}\log n}{n}\right\}.
\end{align*} From~\eqref{eq:bdd_LHS} it follows that $\mathbb{P}\left(\cap_{i \in C_k}\Upsilon_i \right) \leq \mathbb{P}\left( \Sigma_{k} \right).$ For $\Sigma = \cap_{k \in [K]}\Sigma_{k}$, we then get the following bound
\begin{align*}
    \mathbb{P}(\Sigma) \geq 1 - \sum_{ k \in [K]}\mathbb{P}(\Sigma^\mathsf{c}_{k}) &\geq 1 - \sum_{i \in C_k, k \in [K]}\mathbb{P}\left(\Upsilon_i^\mathsf{c}\right) \geq 1 - \frac{2m}{n^3},
\end{align*} where the first inequality follow from the union bound, with the third inequality following from~\eqref{eq:prob_bdd2}. Next, for any $k, l \in [K]$, we have that
\begin{equation}\label{eq:avg_dist}
    \|\overline{\theta}_k - \overline{\theta}_l\| \geq \|\theta_k^\star - \theta^\star_l\| - \|\overline{\theta}_k - \theta^\star_k \| - \|\overline{\theta}_l - \theta^\star_l \|.
\end{equation} Applying Chebyshev's inequality and Lemma~\ref{lm:duchi}, for any $\gamma > 0$ and $k \in [K]$, we get
\begin{align*}
    \mathbb{P}&\big(\|\overline{\theta}_k - \theta_k^\star \| > \gamma \big) \leq \frac{\mathbb{E}\|\overline{\theta}_k - \theta_k^\star \|^2}{\gamma^2} = \mathcal{O}\Bigg(\frac{N_k^2}{n|C_k|\mu_{F_k}^2\gamma^2} + \frac{\lambda_k}{n^2\gamma^2} + \frac{\upsilon_k}{n^2|C_k|\gamma^2} + \frac{\nu_k}{n^3\gamma^2} + \left(\frac{R^2d}{\gamma^2} + \frac{\tau_k}{n\gamma^2}\right)de^{-\alpha_k n}\Bigg)
\end{align*} Define the event $\Lambda = \cap_{k \in [K]}\left\{\omega: \|\overline{\theta}_k - \theta^\star_k \| \leq \gamma \right\}$. We have
\begin{align*}
    \mathbb{P}(\Lambda) \geq 1 &- \sum_{k \in [K]}\mathbb{P}\left(\|\overline{\theta}_k - \theta_k^\star \| > \gamma \right) \geq 1 - \Omega\Bigg(\frac{N^2}{n|C_{(K)}|\mu_F^2\gamma^2} - \frac{\lambda}{n^2\gamma^2} - \frac{\nu}{n^3\gamma^2} \\ &- \frac{\upsilon}{n^2|C_{(K)}|\gamma^2} - \left(\frac{KR^2d}{\gamma^2} + \frac{\tau}{n\gamma^2}\right)de^{-\alpha n}\Bigg),
\end{align*} where $N^2 = \sum_{k \in [K]}N_k^2$, $\lambda = \sum_{k \in [K]}\lambda_k$, $\nu = \sum_{k \in [K]}\nu_k$, $\upsilon = \sum_{k \in [K]}\upsilon_k$, $\tau = \sum_{k \in [K]}\tau_k$, $\mu_F = \min_{k \in [K]}\mu_{F_k}$, $\alpha = \min_{k \in [K]}\alpha_k$. Recall that $D = \min_{k \neq l}\|\theta^\star_k - \theta^\star_l\|$. Plugging everything in~\eqref{eq:avg_dist}, then on $\Lambda$, for any $k, l \in [K]$, $\|\overline{\theta}_k - \overline{\theta}_l\| \geq D - 2\gamma$, for any $\gamma < \frac{D}{2}$. Define the event $\Psi = \Sigma \cap \Lambda$. The following inequalities hold on $\Psi$, for all $i \in C_k$, $k, l \in [K]$: $\|\widehat{\theta}_i - \overline{\theta}_k\| \leq \sqrt{\frac{4M\log n}{n}}$ and $\|\overline{\theta}_k - \overline{\theta}_l\| \geq D - 2\gamma$, with $M = \max_{k \in [K]}M_k$. Plugging these bounds in the recovery condition~\eqref{eq:rec_cond}, we get $\alpha\sqrt{\frac{4M\log n}{n}} < D - 2\gamma,$ implying the true clustering is recovered if
\begin{equation}\label{eq:n_cond}
    \frac{n}{\log n} > \frac{4M\alpha^2}{(D - 2\gamma)^2}.
\end{equation} Therefore, if the number of samples per user satisfies~\eqref{eq:n_cond}, the true clustering can be recovered on $\Psi = \Sigma \cap \Lambda$ and Lemma~\ref{lm:duchi} applies. Using $\mathbb{P}(\Psi) \geq \mathbb{P}(\Sigma) + \mathbb{P}(\Lambda) - 1$, we have
\begin{align*}
    \mathbb{P}(\Psi^\mathsf{c}) \leq \mathcal{O}\Bigg(\frac{N^2}{n|C_{(K)}|\mu_F^2\gamma^2} + \frac{\lambda}{n^2\gamma^2} + \frac{\upsilon}{n^2|C_{(K)}|\gamma^2} \left(\frac{dKR^2}{\gamma^2} + \frac{\tau}{n\gamma^2}\right)de^{-\alpha n} + \frac{m}{n^3}\Bigg) 
\end{align*} Combine everything in~\eqref{eq:setup}, to get
\begin{align*}
    &\mathbb{E}\|\widetilde{\theta}_k - \theta^\star_k\|^2 \leq \mathbb{E}\|\widetilde{\theta}_k - \theta^\star_k \|^2\mathbb{I}_{\Psi} + \mathbb{E}\|\widetilde{\theta}_k - \theta^\star_k \|^2\mathbb{I}_{\Psi^\mathsf{c}} = \mathcal{O}\Bigg(\frac{N_k^2}{n|C_k|\mu_{F_k}^2} + \frac{R^2N^2}{n|C_{(K)}|\mu_F^2\gamma^2}  + \frac{\lambda_k + \nicefrac{\lambda R^2}{\gamma^2}}{n^2} \\ &+ \frac{\upsilon_k}{n^2|C_k|} + \frac{R^2\upsilon}{n^2|C_{(K)}|\gamma^2} + \left(dR^2 + \frac{\tau_k}{n} \right)de^{-\alpha_k n} + \left(\frac{dKR^4}{\gamma^2} + \frac{R^2\tau}{n\gamma^2}\right)de^{-\alpha n} + \frac{\nu_k + \nicefrac{\nu R^2}{\gamma^2} + R^2m}{n^3}\Bigg).
\end{align*} Setting $\gamma = \frac{D}{4}$ completes the proof.
\end{proof}

\subsection{Convex clustering and $K$-means}\label{app:special_algos} 

In this section we describe convex clustering and $K$-means and present proves of Lemmas~\ref{lm:conv_clust_rec} and~\ref{lm:$K$-means_rec}, showing that both convex clustering and $K$-means belong to the class of admissible algorithms.

\subsubsection{Convex clustering}

Convex clustering~\cite{pmlr-v70-panahi17a,hocking2011clusterpath,cvx_clust} is an approach to clustering, formulated as a strongly convex optimization problem, with group lasso regularization. As such, the method is guaranteed to have a unique solution, and, moreover, does not require knowledge of the true number of clusters. Formally, the problem is formulated as
\begin{equation}\label{eq:cvx_clust}
    \argmin_{u_1,\ldots,u_n \in \mathbb{R}^d} \frac{1}{2}\sum_{i \in [n]} \|a_i - u_i\|^2 + \lambda \sum_{i < j} \| u_i - u_j\|,
\end{equation} where $\lambda > 0$ is a tunable parameter. Let $\mathcal{V} = \{V_k\}_{k \in [K]}$ be a partition of the dataset, such that $\cup_{k \in [K]}V_k = \{a_1,\ldots,a_n\}$ and $V_k \cap V_l = \emptyset$, $k \neq l$. Corollary 7 in~\cite{cvx_clust} shows that if
\begin{equation}\label{eq:cond_exact_rec}
    \max_{k \in [K]}\max_{i,j \in V_k}\frac{\|a_i - a_j\|}{|V_k|} \leq \lambda < \min_{\substack{k \neq l \\ k, l \in [K] }}\frac{\|\mu_{V_k} - \mu_{V_l} \|}{2n - |V_k| - |V_l|},
\end{equation} then the clustering is recovered, in the sense that $u^\star_i = u^\star_j,$ for all $i, j \in V_k$ and $u^\star_i \neq u^\star_j$, for all $i \in V_k$, $j \in V_l$, $k \neq l$. Here, $\{u_i^\star\}_{i \in [n]} = \{u_i^\star(\lambda)\}_{i \in [n]}$ is the solution to~\eqref{eq:cvx_clust}. 

We are now ready to prove Lemma~\ref{lm:conv_clust_rec}.

\begin{proof}[Proof of Lemma~\ref{lm:conv_clust_rec}]
    In order to prove the claim, we have to show that condition~\eqref{eq:recovery_condition} guarantees that the interval defined by~\eqref{eq:cond_exact_rec} is non-empty. Notice that for any $i,j \in C_k$, $k \in [K]$
    \begin{equation}\label{eq:mid_step}
        \|a_i - a_j\| \leq \|a_i - \mu_k\| + \|a_j - \mu_k \|, 
    \end{equation} which implies $\max_{i,j \in C_k,k\in[K]}\|a_i - a_j\| \leq 2\max_{i \in C_k,k\in[K]}\|a_i - \mu_k\|$. Setting $\alpha = \frac{4(n - |C_{(K)}|)}{|C_{(K)}|}$ in~\eqref{eq:recovery_condition} and rearranging, we get
    \begin{align*}
        \max_{k \in [K]}\frac{\max_{i,j \in C_k}\|a_i - a_j\|}{|C_k|} \leq \frac{2\max_{i,j\in C_k,k\in[K]}\|a_i - \mu_k\|}{|C_{(K)}|} <  \min_{k\neq l,k,l\in[K]}\frac{\|\mu_k - \mu_l\|}{2n - |C_k| - |C_l|},
    \end{align*} where the first inequality follows from~\eqref{eq:mid_step}, while the second inequality follows from~\eqref{eq:recovery_condition} and the fact that $2(n - |C_{(K)}|) \geq 2n - |C_k| - |C_l|$. Thus, the interval defined by~\eqref{eq:cvx_clust} is non-empty and $\lambda$ can be chosen such that convex clustering recovers the true clusters.
\end{proof}

\begin{remark}
One can also implement a weighted version of convex clustering, given by 
\begin{equation*}
    \frac{1}{2}\sum_{i = 1}^n\|a_i - u_i\|^2 + \lambda \sum_{i < j}w_{ij}\|u_i - u_j\|,
\end{equation*} where $w_{ij} \geq 0$ are weights assigned to the pair $(i,j)$. The weighted version of convex clustering has been observed to perform well in practice, e.g.,~\cite{pmlr-v70-panahi17a,hocking2011clusterpath,splitting_cvx_clust}, and can reduce the complexity of the convex clustering problem, by setting many $w_{ij}$'s to zero. However, recovery guarantees of weighted convex clustering require that across cluster weights for different clusters (i.e., such that $i \in C_k$, $j \in C_l$, $k \neq l$) be non-zero, e.g.,~\cite{cvx_clust}. Therefore, it is not obvious which $w_{ij}$'s can be set to zero without knowing the true clustering structure beforehand. While we chose to present our algorithm using the uniformly-weighted convex clustering method, i.e., $w_{ij} = 1$, for all $i < j$, for the previously outlined reason, we note that a weighted version of convex clustering with a heuristic $k$-Nearest Neighbor weight selection as in, e.g.,~\cite{cvx_clust}, can be implemented with our algorithm. 
\end{remark}

\subsubsection{Spectral $K$-means}

$K$-means is a widely used approach in clustering. Formally, the $K$-means problem is defined as
\begin{equation}\label{eq:Kmeans}
    \argmin_{x_1,\ldots,x_K}\frac{1}{2}\sum_{i \in [n]}\min_{j \in [K]}\|a_i - x_j \|^2.
\end{equation} The problem~\eqref{eq:Kmeans} is non-convex and attaining a globally optimal solution is NP-hard. A simple two-step algorithm, known as Lloyd's algorithm~\cite{Lloyd},\cite{banerjee2005clustering}, is known to produces good solutions in practice and is guaranteed to converge in finite time. The algorithm iteratively updates cluster and center estimates, $x^t_k$ and $C^t_k$ respectively, $k \in [K]$, alternating between the following two steps at time $t+1$:
\begin{itemize}
    \item Cluster update: $C_k^{t+1} = \left\{i \in [n]: \|a_i - x^t_k\| \leq \|a_i - x^t_l\|, \text{ for all } l \in [K] \right\}$,
    \item Center update: $x^{t+1}_k = \frac{1}{|C_k^{t+1}|}\sum_{i \in C^{t+1}_k}a_i$.
\end{itemize} The authors in~\cite{kmeans_awasthi} show that a variation of the Lloyd's algorithm, that uses an approximation algorithm on the SVD space of the data matrix to produce the initial center estimates and afterwards runs Lloyd's algorithm until convergence is guaranteed to recover the true clusters under appropriate conditions. The algorithm is summarized in Algorithm~\ref{alg:alg2}, and we refer to it as \emph{spectral $K$-means}. 

\begin{algorithm}[!t]
   \caption{Spectral $K$-means}
   \label{alg:alg2}
\begin{algorithmic}
   \STATE {\bfseries Part I:} Center initialization:
   \begin{itemize}
       \item Project the data matrix $A \in \mathbb{R}^{n\times d}$ onto the subspace spanned by the top $K$ singular vectors, to get $\widehat{A}$.
       \item Run a $10$-approximation algorithm\footnotemark for the $K$-means problem~\eqref{eq:Kmeans} with the projected matrix $\widehat{A}$ and obtain $K$ centers, $\nu_1,\ldots,\nu_K$.
   \end{itemize}
   \STATE {\bfseries Part II:} Center refinement: 
   \begin{itemize}
       \item Find the ``core'' points with respect to each initial center, given by $S_k = \{i \in [m]: \|\widehat{a}_i - \nu_k\| \leq \frac{1}{3}\|\widehat{a}_i - \nu_l\|, \: \forall l \neq k \}$, for all $k \in [K]$.
        \item Update centers $\nu_k = \mu(S_k)$, for all $k \in [K]$.       
   \end{itemize}
   \STATE {\bfseries Part III:} Run Lloyd's algorithm until convergence:
   \begin{itemize}
       \item Update clusters $S_k = \{i \in [m]: \|a_i - \nu_k \| \leq \| a_i - \nu_l\|, \: \forall l \neq k\}$, for all $k \in [K]$.
       \item Update centers $\nu_k = \mu(S_k)$, for all $k \in [K]$.
   \end{itemize}
\end{algorithmic}
\end{algorithm}
\footnotetext{For $\epsilon > 0$, a $\epsilon$-approximation algorithm for a function $f(x)$ is an algorithm that outputs a solution $x^\star$ such that $f(x^\star) \leq (1 + \epsilon)\min_x f(x)$.}

The conditions that guarantee that spectral $K$-means recovers the true clusters are:
\begin{enumerate}
    \item \emph{Center separation condition}: define the value $\Delta_k = \frac{1}{\sqrt{|C_k|}}\min\{\sqrt{K}\|A - C\|, \|A - C\|_F\}$, $k \in [K]$, we then have
    \begin{equation}\label{eq:Kmeans_cond1}
        \|\mu_k - \mu_l \| \geq c\left(\Delta_k + \Delta_l \right),
    \end{equation} for all $k,l \in [K]$, where $c > 0$ is a global constant.

    \item \emph{Proximity condition}: for every point $i \in C_k$ and every $l \neq k$, the projection of $a_i$ onto the line connecting $\mu_k$ and $\mu_l$, denoted by $P^i_{kl}$, satisfies
    \begin{equation}\label{eq:Kmeans_cond2}
        \|P^i_{kl} - \mu_l \| - \|P^i_{kl} - \mu_k\| \geq c\left(\frac{1}{\sqrt{|C_k|}} + \frac{1}{\sqrt{|C_l|}} \right)\|A - C\|.
    \end{equation}
\end{enumerate}

We are now ready to prove Lemma~\ref{lm:$K$-means_rec}.

\begin{proof}[Proof of Lemma~\ref{lm:$K$-means_rec}]
    We want to show that~\eqref{eq:recovery_condition} guarantees both~\eqref{eq:Kmeans_cond1} and~\eqref{eq:Kmeans_cond2} are satisfied. To that end, we first note 
    \begin{align*}
        \sqrt{m}\max_{i \in C_k,k\in[K]}\|\mu_k - a_i\| &\geq \|A - C\|_F \geq \sqrt{|C_k|}\Delta_k,
    \end{align*} for any $k \in [K]$, therefore we have
    \begin{align*}
        c\left(\Delta_k + \Delta_l \right) &\leq \frac{2c\sqrt{m}}{\sqrt{|C_{(K)}|}}\max_{\substack{i \in C_k \\ k\in[K]}}\|\mu_k - a_i\| < \min_{\substack{k\neq l\\ k,l\in[K]}}\|\mu_k - \mu_l \|,
    \end{align*} hence satisfying~\eqref{eq:Kmeans_cond1}. Next, recall that, for any $k,l \in [K]$, the line passing through $\mu_k$ and $\mu_l$ is parametrized by $(1 - \gamma) \mu_k + \gamma \mu_l$, $\gamma \in \mathbb{R}$, with the line segment between $\mu_k$ and $\mu_l$ parametrized by $\gamma \in [0,1]$. It is easy to see that it suffices to show the condition~\eqref{eq:Kmeans_cond2} is satisfied for all the points $a_i$, $i \in C_k$ whose projection onto the line passing through $\mu_k$ and $\mu_l$ belongs to the line segment between $\mu_k$ and $\mu_l$. For each point whose projection belongs to the specific line segment between $\mu_k$ and $\mu_l$, some simple algebra yields that $\|P^i_{kl} - \mu_l\| = \|\mu_k - \mu_l\| - \|P^i_{kl} - \mu_k \|.$ Therefore, we get
    \begin{align*}
        \|P^i_{kl} - \mu_l \| - \|P^i_{kl} - \mu_k \| \geq \|\mu_k - \mu_l\| - 2\|P^i_{kl} - \mu_k \| \geq \|\mu_k - \mu_l\| - 2\|a_i - \mu_k \|, 
    \end{align*} where in the second inequality we use the fact that projection is a contraction. Using~\eqref{eq:recovery_condition} with $\alpha = 2 + \frac{2c\sqrt{m}}{\sqrt{|C_{(K)}|}}$, we get
    \begin{align*}
        \|P^i_{kl} - \mu_l \| - \|P^i_{kl} - \mu_k \| \geq \frac{2c\sqrt{m}}{\sqrt{|C_{(K)}|}}\max_{i\in C_k,k \in [K]}\|\mu_k - a_i\| \geq c\left(\frac{1}{|C_k|} + \frac{1}{|C_l|} \right)\|A - C\|, 
    \end{align*} hence satisfying~\eqref{eq:Kmeans_cond2}. Thus, both conditions are satisfied and spectral $K$-means is guaranteed to recover the clusters.
\end{proof}

\subsection{Proof of Theorem~\ref{thm:inexact_erm}}\label{app:inexact}

\noindent We begin by stating two result from~\cite{rakhlin_sgd}, used in this section. 

\begin{lemma}[Lemma 1 in~\cite{rakhlin_sgd}]\label{lm:rakhlin_sgd}
Under Assumptions~\ref{asmpt:param_space},~\ref{asmpt:loss_function},~\ref{asmpt:loss_str_cvx} and~\ref{asmpt:stoch_grad}, for all $i \in [m]$, if we set the step-size rule of SGD as $\eta^t = \frac{1}{\mu_{f}t}$, it holds for any $T \geq 1$ and any $i \in [m]$ that $\mathbb{E}\|\theta_i^T - \widehat{\theta}_i \|^2 \leq \frac{4\Gamma_i^2}{\mu_{f}^2T}.$
\end{lemma}

\begin{lemma}[Lemma 2 in~\cite{rakhlin_sgd}]\label{lm:rakhlin_high_prob}
Let Assumptions~\ref{asmpt:param_space},~\ref{asmpt:loss_function},~\ref{asmpt:loss_str_cvx} and~\ref{asmpt:stoch_grad} hold. Then, for all $i \in [m]$ and any $\delta \in (0,\nicefrac{1}{e})$, $T \geq 4$, if we set the step-size rule of SGD as $\eta^t = \frac{1}{\mu_{f}t}$, it holds with probability $1 - \delta$, for any $t \in \{8,\ldots,T-1,T\}$ and any $i \in [m]$, that
\begin{equation*}
    \|\theta^t_i - \widehat{\theta}_i \|^2 \leq \frac{12\Gamma^2}{\mu_{f}^2t} + 8\Gamma(121\Gamma + 1)\frac{\log\left(\nicefrac{\log(t)}{\delta} \right)}{t}.
\end{equation*}
\end{lemma}

We are now ready to state and prove counterparts of Lemmas~\ref{lm:erm} and~\ref{lm:duchi}, when an inexact ERM estimator is used.

\begin{lemma}\label{lm:inexact_sco}
Let Assumptions~\ref{asmpt:clusters}-\ref{asmpt:pop_loss},~\ref{asmpt:unif_bdd_ptws}-\ref{asmpt:stoch_grad} hold. If each user runs SGD locally for $T$ iterations, with the step-size rule $\eta^t = \frac{1}{\mu_{f}t}$, to produce $\theta_i^T$, $i \in [m]$ and $T$ is chosen such that $T \geq 15$ and $\frac{T}{\log\log\left(T \right)} \geq \left(\frac{12\Gamma^2}{\mu_{f}^2} + 8\Gamma(121\Gamma + 1)(1 + \log \frac{1}{\delta})\right)\frac{1}{\varepsilon^2}$, then for any $k \in [K]$, any $i \in C_k$ and any $\epsilon > 0$, $0 < \delta < \frac{1}{3}$, with probability at least $1 - 3\delta$, we have, for any $i \in C_k$, $k \in [K]$
\begin{align*}
    F_k&(\theta^T_i) - F_k(\theta^\star_k) \leq \frac{16R^2LC(\epsilon,\delta)}{n} + \frac{8LF_k(\theta^\star_k)\log \frac{2}{\delta}}{\mu_{F_k} n} + \frac{8RN_k\log \frac{2}{\delta}}{n} + \left(8RL + G_{F_k} + \frac{4RLC(\epsilon,\delta)}{n} \right)\epsilon + \varepsilon S,
\end{align*} where $C(\epsilon,\delta) \coloneqq 2\left(\log\frac{2}{\delta} + d\log\frac{6R}{\epsilon}\right)$. 
\end{lemma}
\begin{proof}
For any $\theta \in \Theta$, any $k \in [K]$ and any $i \in C_k$, we have
\begin{equation}\label{eq:bounding_F}
    F_k(\theta) - F_k(\theta^\star_k) \leq \left|F_k(\theta) - F_k(\widehat{\theta}_i) \right| + F_k(\widehat{\theta}_i) - F_k(\theta^\star_k).
\end{equation} We can bound the second term on the right hand side of~\eqref{eq:bounding_F} using Lemma~\ref{lm:erm}. To bound the first term, we use Lipschitz continuity of $F_k$ (recall the discussion in Section~\ref{sec:problem}), to get
\begin{equation}\label{eq:bounding_F_with_f}
    \left|F_k(\theta) - F_k(\widehat{\theta}_i) \right| \leq G_{F_k}\|\theta - \widehat{\theta}_i \|.
\end{equation} Applying Lemma~\ref{lm:rakhlin_high_prob}, we have, with probability at least $1 - \delta$
\begin{equation*}
    \|\theta^T_i - \widehat{\theta}_i \|^2 \leq \frac{12\Gamma^2}{\mu_{f}^2T} + 8\Gamma(121\Gamma + 1)\frac{\log\left(\nicefrac{\log(T)}{\delta} \right)}{T}.
\end{equation*} As $T \geq 15$, we get, with probability at least $1 - \delta$
\begin{align*}
    \|\theta^T_i - \widehat{\theta}_i \|^2 &\leq \beta_\delta\frac{\log\log(T)}{T},
\end{align*} where $\beta_\delta = \left(\frac{12\Gamma^2}{\mu_{f}^2} + 8\Gamma(121\Gamma + 1)\left(1 + \log \frac{1}{\delta}\right)\right)$. From the conditions of the lemma, we then have $\|\theta^T_i - \widehat{\theta}_i \| \leq \varepsilon.$ Combining~\eqref{eq:bounding_F} and~\eqref{eq:bounding_F_with_f}, with probability at least $1 - \delta$
\begin{equation*}
    F_k(\theta^T_i) - F_k(\theta^\star_k) \leq \varepsilon G_{F_k} + F_k(\widehat{\theta}_i) - F_k(\theta^\star_k).
\end{equation*} The result is completed by applying Lemma~\ref{lm:erm} to the second term on the right hand side of the final inequality.
\end{proof}

\begin{lemma}\label{lm:inexact_duchi}
Let Assumptions~\ref{asmpt:clusters}-\ref{asmpt:stoch_grad} hold and each user runs SGD locally for $T$ iterations, to produce $\theta_i^T$. If $T \geq \frac{4\Gamma^2}{\mu_{f}^2\varepsilon}$, then for $\widetilde{\theta}_{k} = \frac{1}{|C_k|}\sum_{i \in C_k}\theta^T_i$, $k \in [K]$, we have
\begin{align*}
    \mathbb{E}\|\widetilde{\theta}_{k} - \theta^\star_k \|^2 = \mathcal{O}\Bigg(\frac{N_k^2}{n|C_k|\mu_{F_k}^2} + \frac{\lambda_k}{n^2} + \frac{\upsilon_k}{n^2|C_k|} + \frac{\nu_k}{n^3} + \left(dR^2 + \frac{\tau_k}{n} \right)de^{-\alpha_k n}\Bigg) + \varepsilon.
\end{align*}
\end{lemma}
\begin{proof}
From Lemma~\ref{lm:rakhlin_sgd}, we know that, for each $i \in [m]$, running SGD locally for $T \geq \frac{4\Gamma^2}{\mu_{f}^2\varepsilon}$ iterations results in $\mathbb{E}\|\theta_i^T - \widehat{\theta}_i \|^2 \leq \varepsilon$. Define the across-cluster average of $\varepsilon$-inexact approximations as $\widetilde{\theta}_k = \frac{1}{|C_k|}\sum_{i \in C_k}\theta^T_i$. We then have
\begin{equation}\label{eq:young's}
    \mathbb{E}\|\widetilde{\theta}_k - \theta_k^\star \|^2 \leq 2\mathbb{E}\|\overline{\theta}_k - \theta_k^\star \|^2 + 2\mathbb{E}\|\widetilde{\theta}_k - \overline{\theta}_k \|^2, 
\end{equation} where $\overline{\theta}_k = \frac{1}{|C_k|}\sum_{i \in C_k}\widehat{\theta}_i$. We can bound the first term on the right-hand side of~\eqref{eq:young's} using Lemma~\ref{lm:duchi}. For the second term, we have $\mathbb{E}\|\widetilde{\theta}_k - \overline{\theta}_k \|^2 \leq \frac{1}{|C_k|}\sum_{i \in C_k}\mathbb{E}\|\theta^T_i - \widehat{\theta}_i \|^2 = \varepsilon$. Combining the results and plugging in~\eqref{eq:young's} completes the proof.
\end{proof}

Lemmas~\ref{lm:inexact_sco} and~\ref{lm:inexact_duchi} are the counterparts of Lemmas~\ref{lm:erm} and~\ref{lm:duchi}, when the local ERMs are solved approximately. The proof of Theorem~\ref{thm:inexact_erm} uses similar arguments as the proof of Theorem~\ref{thm:main}, replacing the results of Lemmas~\ref{lm:erm} and~\ref{lm:duchi} with results from Lemmas~\ref{lm:inexact_sco} and~\ref{lm:inexact_duchi}. For the sake of brevity, we omit the proof.

\subsection{Proof of Theorem~\ref{thm:inexact_clust}}\label{app:inexact_clust}

\noindent Let $S$ be the event that the clustering algorithm $A$ returned a clustering $\{C_k^\prime \}_{k \in [K]}$, such that $|C_k^\prime \cap C_k| = (1 - \epsilon)|C_k^\prime|$, for all $k \in [K]$, where $\{C_k \}_{k \in [K]}$ is the true underlying clustering. Denote by $\epsilon_{kl} = \frac{|C_k^\prime \cap C_l|}{|C_k^\prime|}$, for all $l \neq k$ and notice that $\epsilon = \sum_{l \neq k}\epsilon_{kl}$. We then have the following result.

\begin{lemma}\label{lm:inexact_clust}
    Let $S$ denote the event that the inexact clustering $\{C_k^\prime \}_{k \in [K]}$ was returned and denote the cluster models by $\widetilde{\theta}_k = \frac{1}{|C_k^\prime|}\sum_{i \in C_k^\prime}\widehat{\theta}_i$. Then, the conditional MSE, conditioned on the even $S$, for any $k \in [K]$, is given by
    \begin{equation*}
        \mathbb{E}\left[\|\widetilde{\theta}_k - \theta^\star_k \|^2 \big\vert S \right] = (K-1)\mathcal{O}\left(\frac{1}{n|C_k^\prime|} + \frac{1}{n^2} + \sum_{l \neq k}\epsilon_{kl}^2D_{kl} \right).
    \end{equation*}
\end{lemma}
\begin{proof}
    Note that the model $\widetilde{\theta}_k$ can be rewritten as
    \begin{equation}\label{eq:imp_upd}
        \widetilde{\theta}_k = \frac{1}{|C_k^\prime|}\left( \sum_{i \in C_k^\prime \cap C_k}\widehat{\theta}_i + \sum_{l \neq k}\sum_{j \in C_k^\prime \cap C_l}\widehat{\theta}_j \right).
    \end{equation} Using~\eqref{eq:imp_upd}, and some simple algebraic manipulations yields that, for any $k \in [K]$, $\widetilde{\theta}_k - \theta^\star_k = A_1 + A_2$, where
    \begin{align*}
        A_1 = \frac{1-\epsilon}{|C_k^\prime \cap C_k|}\sum_{i \in C_k^\prime \cap C_k}\left(\widehat{\theta}_i - \theta^\star_k \right) \text{ and } A_2 = \sum_{l \neq k}\frac{\epsilon_{kl}}{|C_k^\prime \cap C_l|}\sum_{j \in C_k^\prime \cap C_l}\left(\widehat{\theta}_j  - \theta^\star_k\right).
    \end{align*} Using $(a + b)^2 \leq 2a^2 + 2b^2$ and the decomposition above, we have that the conditional MSE is given by
    \begin{align*}
        \mathbb{E}\left[\|\widetilde{\theta}_k - \theta^\star_k \|^2 \big\vert S \right] \leq 2\mathbb{E}\left[\|A_1\|^2 \big\vert S \right] + 2\mathbb{E}\left[\|A_2\|^2 \big\vert S \right].
    \end{align*} We analyze the two terms on the right-hand side separately. For the ease of notation, let the conditional expectation on an event $S$ be denoted by $\mathbb{E}_S[\cdot]$, i.e., $\mathbb{E}_S[\cdot] = \mathbb{E}[\cdot \vert S]$. We have
    \begin{align*}
        \mathbb{E}\left[\|A_1 \|^2 \big\vert S\right] = \mathbb{E}_S\left\| \frac{(1 - \epsilon)}{|C_k^\prime \cap C_k|} \sum_{i \in C_k^\prime \cap C_k}\left(\widehat{\theta}_i - \theta^\star_k\right) \right\|^2 &= (1 - \epsilon)^2\mathcal{O}\left(\frac{1}{(1-\epsilon)|C_k^\prime|n} + \frac{1}{n^2} \right) \\ &= \mathcal{O}\left(\frac{1-\epsilon}{|C_k^\prime|n} + \frac{(1-\epsilon)^2}{n^2} \right), 
    \end{align*} where in the second equality we applied Lemma~\ref{lm:duchi}. Next,
    \begin{align*}
        \mathbb{E}\left[\|A_2 \|^2 \big\vert S\right] = \mathbb{E}_S\left\| \sum_{l \neq k}\frac{\epsilon_{kl}}{|C_k^\prime \cap C_l|} \sum_{i \in C_k^\prime \cap C_l}\left(\widehat{\theta}_i - \theta^\star_k\right)\right\|^2 &= \mathbb{E}_S\left\| \sum_{l \neq k}\frac{\epsilon_{kl}}{|C_k^\prime \cap C_l|} \sum_{i \in C_k^\prime \cap C_l}\left(\widehat{\theta}_i - \theta^\star_k \pm \theta^\star_l\right)\right\|^2 \\ &\leq 2\mathbb{E}\left[\| A_3 \|^2 \big\vert S \right] + 2\mathbb{E}\left[\|A_4\|^2 \big\vert S\right],
    \end{align*} where
    \begin{align*}
        A_3 = \sum_{l \neq k}\frac{\epsilon_{kl}}{|C_k^\prime \cap C_l|} \sum_{i \in C_k^\prime \cap C_l}\left(\widehat{\theta}_i - \theta^\star_l\right) \text{ and } A_4 = \sum_{l \neq k}\epsilon_{kl}\left(\theta^\star_l - \theta^\star_k\right).
    \end{align*} We look at the terms $A_3$ and $A_4$ separately. We have
    \begin{align*}
        \mathbb{E}\left[\| A_3 \|^2 \big\vert S \right] = \mathbb{E}_S\left\|\sum_{l \neq k}\frac{\epsilon_{kl}}{|C_k^\prime \cap C_l|}\sum_{i \in C_k^\prime \cap C_l}\widehat{\theta}_j - \theta_l^\star \right\|^2 &\leq (K-1)\sum_{l \neq k}\epsilon_{kl}^2\mathbb{E}\left\|\frac{1}{|C_k^\prime \cap C_l|}\sum_{j \in C_k^\prime \cap C_l}\widehat{\theta}_j - \theta^\star_l\right\|^2 \\ &= \mathcal{O}\left(\frac{(K-1)\epsilon}{|C_k^\prime|n} + (K-1)\sum_{l \neq k}\frac{\epsilon_{kl}^2}{n^2} \right), 
    \end{align*} where we used $\|\sum_{i = 1}^KX_i \|^2 \leq K\sum_{i = 1}^K\|X_i \|^2$ in the first inequality and Lemma~\ref{lm:duchi} in the second equality. Next, we have
    \begin{align*}
        \mathbb{E}\left[\| A_4 \|^2 \big\vert S \right] = \Big\|\sum_{l \neq k}\epsilon_{kl}(\theta_l^\star - \theta_k^\star) \Big\|^2 \leq (K-1)\sum_{l \neq k}\epsilon_{kl}^2D_{kl},
    \end{align*} where $D_{kl} = \|\theta^\star_k - \theta^\star_l\|^2$ is the distance of population optimal models between clusters $k$ and $l$. Combine everything, to get
    \begin{align*}
        \mathbb{E}_S\|\widetilde{\theta}_k - \theta^\star_k \|^2 \leq 2\mathbb{E}_S\|A_1\|^2 + 4\mathbb{E}_S\|A_3\|^2 + 4\mathbb{E}_S\|A_4\|^2 = (K-1)\mathcal{O}\left(\frac{1}{n|C_k^\prime|} + \frac{1}{n^2} + \sum_{l \neq k}\epsilon_{kl}^2D_{kl} \right),
    \end{align*} which completes the proof.
\end{proof}

\begin{remark}
    The MSE rate in Lemma~\ref{lm:inexact_clust} can be separated in two parts. The first part, $\mathcal{O}\left(\frac{1}{n|C_k|} + \frac{1}{n^2}\right)$, vanishes asymptotically, while the second part, $\sum_{l \neq k}\epsilon_{kl}^2D_{kl}$, is non-vanishing, depending on the fraction of misclustered points and heterogeneity.  
\end{remark}

Before proving Theorem~\ref{thm:inexact_clust}, we state a result from~\cite{kmeans_awasthi}.

\begin{lemma}[Theorem 3.1 in~\cite{kmeans_awasthi}]\label{lm:awasthi}
Let the center separation condition from Definition~\ref{def:cent_sep} hold and let $\{C_k^\prime \}_{k \in [K]}$ be the clustering produced by running Part I of the spectral $K$-means algorithm, with $\{C_k \}_{k \in [K]}$ being the true clustering. Then, $\{C_k^\prime\}_{k \in [K]}$ satisfies the following properties, for any $k \in [K]$
\begin{enumerate}
    \item $|C_k^\prime \cap C_k| = \left(1 - \frac{128}{c^2}\right)|C_k|$,
    \item $\sum_{l \neq k}|C_k^\prime \cap C_l| = \frac{128}{c^2}|C_k|$
\end{enumerate}
\end{lemma}

\begin{proof}[Proof of Theorem~\ref{thm:inexact_clust}]
    Note that it suffices to show that Definition~\ref{def:cent_sep} holds with sufficiently high probability. Indeed, if Definition~\ref{def:cent_sep} holds, properties $1)$ and $2)$ readily follow from Lemma~\ref{lm:awasthi}, while part $3)$ follows by applying Lemma~\ref{lm:inexact_clust}, conditioning on the event that a clustering of sufficient quality is produced (guaranteed by $1)$ and $2)$). The idea is similar to the one in proof of Theorem~\ref{thm:main} and instead of going through the whole proof, we only highlight the key differences.  

    First, note that the condition for cluster separation used in Theorem~\ref{thm:main} was $\alpha \max_{\substack{i \in C_k, \\ k \in [K]}}\| \widehat{\theta}_i - \overline{\theta}_k\| \leq \min_{\substack{k \neq l, \\ k,l \in [K]}}\|\overline{\theta}_k - \overline{\theta}_l \|$. However, via Definition~\ref{def:cent_sep}, we instead use the following condition 
    \begin{equation}\label{eq:max-avg}
        c^2\left(\frac{1}{\sqrt{|C_k|}} + \frac{1}{\sqrt{|C_l|}} \right)^2\sum_{i \in [m]}\|\widehat{\theta}_i - \overline{\theta}_{(i)} \|^2 \leq \min_{k \neq l}\|\overline{\theta}_k - \overline{\theta}_l\|^2,
    \end{equation} where $\overline{\theta}_{(i)} = \overline{\theta}_k$, if $i \in C_k$. Therefore, we can rewrite
    \begin{equation}\label{eq:max-avg1}
        \sum_{i \in [m]}\|\widehat{\theta}_i - \overline{\theta}_{(i)} \|^2 = \sum_{k \in [K]}\sum_{i \in C_k}\|\widehat{\theta}_i - \overline{\theta}_k \|^2
    \end{equation} Next, note that, for any $i \in C_k$, $k \in [K]$, we have
\begin{equation*}
\begin{aligned}
    \|\widehat{\theta}_i - \overline{\theta}_k \|^2 \leq 2\|\widehat{\theta}_i - \theta^\star_k \|^2 + \frac{2}{|C_k|}\sum_{j \in C_k}\|\widehat{\theta}_j - \theta^\star_k \|^2,
\end{aligned}
\end{equation*} which implies
\begin{equation}\label{eq:max-avg2}
    \sum_{i \in C_k}\|\widehat{\theta}_i - \overline{\theta}_k \|^2 \leq 4\sum_{i \in C_k}\|\widehat{\theta}_i - \theta^\star_k \|^2
\end{equation} Now, similar to the proof of Theorem~\ref{thm:main}, we can use the results on the statistical convergence of empirical risk minimizers, to conclude that, with probability at least $1 - \frac{2}{n^3}$, $\|\widehat{\theta}_i - \theta^\star_k \|^2 \leq \frac{M_{ik}\log n}{n}$, where $M_{ik}$ is defined in the proof of Theorem~\ref{thm:main}. Combining~\eqref{eq:max-avg1} and~\eqref{eq:max-avg2}, plugging into~\eqref{eq:max-avg} and using analogous arguments to the ones in proof of Theorem~\ref{thm:main}, one can show that in order to guarantee that the center separation condition (Definition~\ref{def:cent_sep}) holds, one has the following bound $\frac{2c\sqrt{m}}{\sqrt{|C_{(K)}|}}\sqrt{\frac{4\overline{M}\log n}{n}} < \frac{D}{2}$, where $\overline{M} = \frac{1}{m}\sum_{\substack{i \in C_k, \\ k \in [K]}}M_{ik}$, which leads to the sample complexity $\frac{n}{\log n} > \frac{64c^2m\overline{M}}{D^2|C_{(K)}|}.$ Since we consider the same events as in Theorem~\ref{thm:main}, the probability of the complement is the same as that in Theorem~\ref{thm:main}, hence, combining with Lemma~\ref{lm:inexact_clust}, the result follows. Next, using parts $1)$ and $2)$, we know that $\left(1 - \frac{128}{c^2}\right)|C_k| \leq |C_k^\prime| \leq \left(1 + \frac{128}{c^2} \right)|C_k|$. Combining this inequality and part $1)$, it can be shown that $\epsilon \leq \frac{256}{c^2 + 128}$. Similarly, combining this inequality and part $2)$, it can be shown that $\epsilon \leq \frac{128}{c^2 - 128}$. Since $\frac{128}{c^2 - 128} \leq \frac{256}{c^2 + 128}$, whenever $c \geq 8\sqrt{6}$ and the constant $c$ is assume to be $c \geq 100$, the final claim follows.
\end{proof}

\section{Practical considerations}\label{app:practical}

In this section we provide some practical considerations when implementing algorithms from the ODCL-$\mathcal{C}$ family. Specifically, we discuss how to choose hyperparamters for both ODCL-CC and ODCL-KM, and we also describe a simplified version of ODCL-KM, dubbed ODCL-KM++.

\textbf{ODCL-CC:} Notice that the upper bound for $\lambda$ in~\eqref{eq:cond_exact_rec} requires the knowledge of the sizes of the two smallest clusters. Moreover, the lower and upper bounds in the recovery condition~\eqref{eq:cond_exact_rec} both depend on the recovered clustering, which in turns depends on the value of $\lambda$, via~\eqref{eq:cvx_clust}. This shows that~\eqref{eq:cond_exact_rec} can only be verified in ``a posteriori" manner, after~\eqref{eq:cvx_clust} is solved. Therefore, choosing an appropriate value of $\lambda$ can be difficult in practice. In practice, an appropriate value of the hyperparameter $\lambda$ can be chosen as follows:
\begin{enumerate}
    \item The server receives the local models $\{\widehat{\theta}_i \}_{i=1}^m$ and chooses a range of strictly increasing values of $\lambda$, $\left\{\lambda_1,\lambda_2,\ldots,\lambda_N\right\}$, such that solving the convex clustering problem~\eqref{eq:cvx_clust} results in the number of clusters $K_{\lambda_i}$ satisfying $K_{\lambda_1} = m$ and $K_{\lambda_N} = 1$.\footnote{From the formulation of convex clustering~\eqref{eq:cvx_clust}, it is obvious that, for $\lambda$ sufficiently small, the optimal solution is going to be $u_i^\star = a_i$, $i \in [m]$, i.e., $K_\lambda = m$. On the other hand, the authors in~\cite{bianchi_robust_consensus} show that, for $\lambda$ sufficiently large, we have $K_\lambda = 1$. Hence, the choices of $\lambda$ guaranteeing $K_{\lambda} = m$ and $K_{\lambda} = 1$ always exist.}

    \item The server runs convex clustering for each value of $\lambda_i$ and verifies the following:
    \begin{enumerate}
        \item if the condition~\eqref{eq:cond_exact_rec} is verified for some values of $\lambda_i$, the server takes a value of $\lambda_i$ (and the associated clustering) such that the same number of clusters $K_{\lambda_i}$ is recovered for the largest number of $\lambda_i$'s, among all of $\lambda_i$'s, $i = 1,\ldots,N$;
        \item if the condition~\eqref{eq:cond_exact_rec} is not verified for any value of $\lambda_i$, the server takes a value $\lambda_i$ (and the associated clustering) such that the same number of clusters $K_{\lambda_i}$ is recovered for the largest number of $\lambda_i$'s, among all of $\lambda_i$'s, $i = 1,\ldots,N$.
    \end{enumerate}
\end{enumerate} 

The procedure is designed in the spirit of ``clusterpath", e.g.,~\cite{hocking2011clusterpath}. The intuition behind the procedure is to choose the clustering that is the likeliest to be true. If~\eqref{eq:cond_exact_rec} is verified for some $\lambda$'s, the likelihood of the recovered clustering matching the true one is increasing. Note that in general, the recovery guarantees of convex clustering hold only when $\lambda$ satisfies~\eqref{eq:cond_exact_rec}. However, in practice, convex clustering is known to perform well even when the condition~\eqref{eq:cond_exact_rec} is not met, e.g.,~\cite{cvx_clust} show that exact clustering can be recovered even for values of $\lambda$ not in~\eqref{eq:cond_exact_rec}, with, e.g.,~\cite{hocking2011clusterpath,splitting_cvx_clust}, validating the performance on real data, without the knowledge of~\eqref{eq:cond_exact_rec}. We provide numerical experiments that verify these observations in Appendix~\ref{subsec:clusterpath}.    

\textbf{ODCL-KM}: Notice that running ODCL-KM in practice requires the knowledge of the true number of clusters $K$. This information is typically not known beforehand. We can estimate the number of clusters using the following general approach:
\begin{enumerate}
        \item The server receives the local models $\{\widehat{\theta}_i \}_{i \in [m]}$ and chooses an increasing range of values for parameter $K$, $K_i = i$, for $i \in [N]$. 
        
        \item The server runs $K$-means for each value of $K_i$ and evaluates the quality of the produced clusters by using some metric.

        \item Choose the optimal value of $K$ to be the value of $K_i$ that leads to the best quality of produced clusters, based on the metric of choice. 
\end{enumerate} In practice, there are many metrics that can be used to evaluate the quality of the clustering. Here, we list two simple and very popular such metrics:
\begin{itemize}
    \item Decrease in the $K$-means cost obtained by increasing the number of clusters from $K$ to $K + 1$. Using this metric results in the approach known as \emph{elbow method}. For example, one can use this approach and select $K$ to be the first value of $K_i$ for which the cost of adding another cluster improves less than a predefined threshold.

    \item Silhouette score of the clustering. Silhouette score measures the quality of clustering, by estimating how well each object fits its cluster. The reader is referred to~\cite{kaufman2009finding}, for a detailed description of the silhouette score metric. In this case, since a higher silhouette score indicates better quality clusters, we choose $K$ to be the value of $K_i$ which maximizes the silhouette score.
\end{itemize}

\textbf{ODCL-KM++}: Notice that recovery guarantees of spectral $K$-means from the previous section require a computationally expensive
step of projecting the data onto the SVD space and running an approximation algorithm. In order to streamline the process, in place of parts I and II of Algorithm~\ref{alg:alg2}, we use the $K$-means++ scheme~\cite{kmeans++}, leading to algorithm dubbed ODCL-KM++. The $K$-means++ algorithm is summarized in Algorithm~\ref{alg:alg3}. We use ODCL-KM++ in numerical experiments in Section 5, where it is shown to perform optimally, recovering the true clusters even in the low sample regimes.

\begin{algorithm}[!t]
   \caption{$K$-means++}
   \label{alg:alg3}
\begin{algorithmic}
   \STATE {\bfseries Step I:} Center initialization:
   \begin{itemize}
       \item Choose the initial center $\nu_1 = a_i$, by drawing $i \in [n]$, uniformly at random.
       \item For $k = 2,\ldots,K$:
       \begin{enumerate}
           \item Compute $D_{i,k} = \min_{\{\nu_1,\ldots,\nu_{k-1} \}}\|a_i - \nu_k \|$.
           \item Choose the next center $\nu_k = a_i$, with probability $\frac{D_{k,i}^2}{\sum_{i \in [n]}D_{k,i}^2}$, for all $i \in [n]$. 
       \end{enumerate}
            
   \end{itemize}
   \STATE {\bfseries Step II:} Run Lloyd's algorithm until convergence:
   \begin{itemize}
       \item Update clusters $S_k = \{i \in [m]: \|a_i - \nu_k \| \leq \| a_i - \nu_l\|, \: \forall l \neq k\}$, for all $k \in [K]$.
       \item Update centers $\nu_k = \mu(S_k)$, for all $k \in [K]$.
   \end{itemize}
\end{algorithmic}
\end{algorithm}

\section{Additional Experiments}\label{app:numerical}

In this section we present additional numerical experiments. Subsection~\ref{subsec:logistic} presents results using logistic regression loss. Subsection~\ref{subsec:clusterpath} presents the comparison of our two methods proposed in subsections~3.1 and~3.2 from the main paper. Subsection~\ref{subsec:IFCA} presents comparison of our method with IFCA. The method used in the experiments throughout this section is ODCL-CC.

\subsection{Logistic regression}\label{subsec:logistic}

In this subsection, we consider a logistic regression problem. The logistic regression model assumes that the data is generated as follows: for each pair $(x_{ij},y_{ij})_{j = 1}^n$, $i \in C_k$, the label $y_{ij}$ is generated according to $y_{ij} = 2\text{Bernoulli}(p_{ij}) - 1$, where $\text{Bernoulli}(p_{ij})$ is a sample from the Bernoulli distribution with parameter $p_{ij}$. This leads to $y_{ij} \in \{-1,+1\}$, while we compute the probabilities $p_{ij}$ by
\begin{equation*}
    p_{ij} = \frac{1}{1 + \exp{\left(-\left(\langle x_{ij},\theta^\star_k \rangle + b^\star_k\right) \right)}}.
\end{equation*} The number of clusters is set to $K = 4$. The vectors $\theta_k^\star$ are $d$-dimensional, with $d = 2$ (corresponding to the weight parameters) with $b^\star_k$ a scalar (corresponding to the intercept). Specifically, we chose $\theta_k^\star$'s as: $\theta_{1}^\star = \begin{bmatrix} 1 & -1 \end{bmatrix}^T$, $\theta_{2}^\star =\begin{bmatrix} 1 & 0 \end{bmatrix}^T$, $\theta_{3}^\star =\begin{bmatrix} -1 & 1 \end{bmatrix}^T$ and $\theta_{4}^\star =\begin{bmatrix} 0 & -1 \end{bmatrix}^T$, with $b^\star_k = 0$, for all $k \in [K]$. Such a choice of $\theta^\star_k$'s ensures that $D > 0$. Each cluster is assigned a total of $N_k = 100000$ points, with the datapoints $x$ all centered at $\mu = \begin{bmatrix} 0 & 0\end{bmatrix}^T$, with covariance matrices $\Sigma_1 = \begin{bmatrix} 1 & 0 \\ 0 & 1 \end{bmatrix}$, $\Sigma_2 = \begin{bmatrix} 2 & 1 \\ 1 & 2 \end{bmatrix}$, $\Sigma_3 = \begin{bmatrix} 1 & 2 \\ 2 & 1 \end{bmatrix}$ and $\Sigma_4 = \begin{bmatrix} 2 & 0 \\ 0 & 2 \end{bmatrix}$, corresponding to clusters $C_1$ through to $C_4$,
respectively.

To measure the error, we use the $\ell_2$ regularized logistic loss, i.e.,
\begin{equation*}
    \ell\left((x,y);\theta,b\right) = \log\left(1 + \exp\left(-y\left(\langle x,\theta \rangle + b\right) \right)\right) + \frac{C\|\theta\|^2}{2},
\end{equation*} where $C > 0$ is a regularization parameter. In our experiments, we set $C = 10^{-5}$. Under the proposed loss, we have that $\theta^\star_k$'s are the population optimal models, i.e., $\theta^\star_k = \argmin_{u}F_k(u)$, $k \in [K]$.

We consider a FL system with $m = 100$ users and a balanced clustering, i.e., $|C_k| = \frac{m}{K} = 25$, for all $k \in [K]$. Each user $i \in C_k$ is assigned $n$ points uniformly at random, from the corresponding sample $N_k$, such that no data point is assigned to two different users, again simulating an IID distribution of data within clusters. We use the same benchmarks, error metric and selection procedure for $\lambda$, as described in Section~\ref{sec:numerical} of the main paper. When implementing ODCL-CC, to select the parameter $\lambda$, we first compute the lower and upper bounds in~\eqref{eq:cond_exact_rec}. If the lower bound is strictly smaller than the upper bound, we choose $\lambda$ uniformly at random from the interval in~\eqref{eq:cond_exact_rec}. Otherwise, for simplicity, we take $\lambda$ to be equal to the computed upper bound. All results were averaged across 20 runs.

\begin{figure}[ht]
\centering
\begin{tabular}{ll}
\includegraphics[width=0.49\columnwidth]{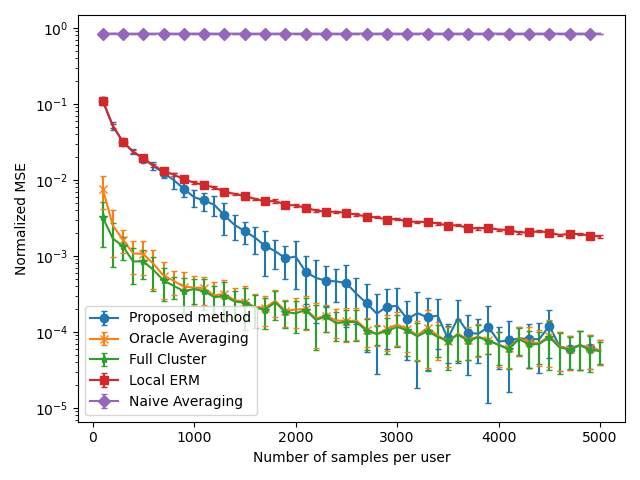}
&
\includegraphics[width=0.49\columnwidth]{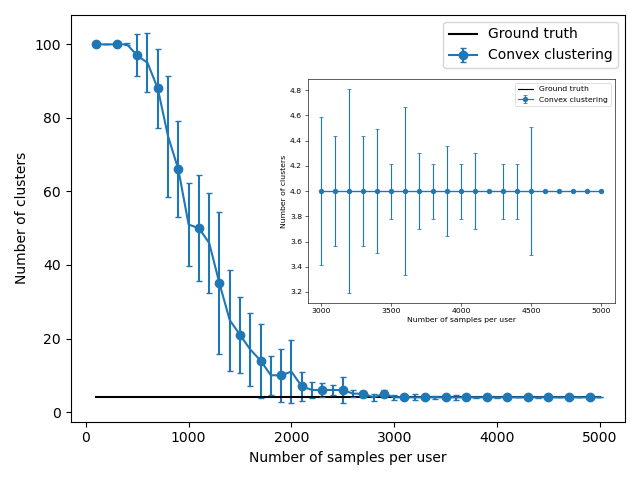}
\end{tabular}
\caption{\emph{Left}: Performance of different methods for logistic regression, versus the number of samples available per user. We can see that the proposed method matches the order-optimal MSE rates for a sufficiently large sample size. \emph{Right}: Number of clusters produced by convex clustering for logistic regression, versus the number of samples available per user. We can see that convex clustering is able to recover the exact clustering for a sufficiently large sample size available to each user.}
\label{fig:MSE_logistic}
\end{figure}

Figure~\ref{fig:MSE_logistic} presents the performance using logistic regression models. The left figure presents the MSE of different methods, while the right figure presents the number of clusters produced by convex clustering. On the $y$-axis in the left figure we plot the averaged normalized MSE, while on the $x$-axis we present the number of samples $n$ available to each user. We can see that, for a small number of samples (less than $500$), ODCL-CC clusters each user to an individual cluster, effectively performing like the local ERMs. As $n$ grows, we see that the quality of our estimator increases, eventually matching the performance of the order-optimal oracle methods (for $n \geq 4600$), as predicted by the theory. The difference in the number of samples required for reaching order-optimal rates of our proposed method and the Oracle Averaging is much larger compared to the linear regression problem. This can be explained by the fact that the sample requirements of Oracle Averaging depend only on the size of the largest cluster, while the sample requirements of the proposed method depend on the problem parameters via the value $M$ (recall Theorem~\ref{thm:main}), as well as the size of the largest cluster. As logistic regression is a more complex and more ``nonlinear" model compared to linear regression, the sample requirements of ODCL-CC increase accordingly. On $y$-axis in the right figure we plot the number of clusters produced by the convex clustering algorithm. On $x$-axis we again plot the number of samples $n$. The right figure is consistent with the results from the left left figure, as it shows that, for small $n$ (less than $500$), convex clustering clusters each user separately, which, due to the low sample regime and our sub-optimal choice of $\lambda$, is to be expected. As the number of samples grows, convex clustering produces lower number of clusters. Zooming in on the interval $3000 \leq n \leq 5000$ (subplot in right figure), we can see that after reaching $3000$ samples, convex clustering will on average produce $K^\prime = 4$ models (consistent with the ground truth), with some variations. This is again consistent with the figure on the left of Figure~\ref{fig:MSE_logistic}, as we see that for larger $n$ our method is getting closer to the performance of the two oracle methods, exactly matching them for $n \geq 4600$, when, as we can see from from the figure on the right, convex clustering is consistently producing $K^\prime = 4$ clusters. Again, we can see that the clustering produced by the convex clustering method is correct, as our method closes in on the performance of both oracle methods that know the true clustering, eventually matching them for $n$ large enough.

\subsection{Clusterpath}\label{subsec:clusterpath}

In this subsection, we evaluate the performance of a version of ODCL-CC that uses clusterpath, described in~\ref{app:practical}, and compare it to the performance of ODCL-CC that was used in Section~\ref{sec:numerical} of the main body. We consider a linear regression problem, with $K = 4$. The experiment design is identical to the one in Section~\ref{sec:numerical} of the main body, with the optimal models chosen as $u^\star_{1i}  \sim \text{U}([0,1])$, $u^\star_{2i} \sim \text{U}([1,2])$, $u^\star_{3i} \sim \text{U}([-1,0])$ and $u^\star_{4i} \sim \text{U}([-2,-1])$, $i = 1,\ldots,d$, with $d = 20$.

We again consider a FL system with $m = 100$ users and a balanced clustering, i.e., $|C_k| = \frac{m}{K} = 10$, for all $k \in [K]$. Each user $i \in C_k$ is assigned $n$ points uniformly at random, from the corresponding sample $N_k$, such that no data point is assigned to two different users. We compare the following two methods:
\begin{itemize}
    \item \emph{Exact convex clustering} - ODCL-CC;
    \item \emph{Clusterpath} - ODCL-CC using clusterpath, described in~\ref{app:practical}.
\end{itemize} 

To find $\lambda_1$ such that $K_{\lambda_1} = 1$ and $\lambda_N$ such that $K_{\lambda_N} = m$, we initialize the values $\lambda_1$, $\lambda_N$ randomly, then run the convex clustering algorithm and check the resulting number of clusters. If $K_{\lambda_1} > 1$, we set $\lambda_1 = \alpha \lambda_1$ and similarly, if $K_{\lambda_N} < m$, we set $\lambda_N = \nicefrac{\lambda_N}{\alpha}$. Once $\lambda_1$ and $\lambda_N$ such that $K_{\lambda_1} = 1$ and $K_{\lambda_N} = m$ are found, we choose $N$ equidistant points from the interval $[\lambda_N,\lambda_1]$. In our experiments, we used the initialization $\lambda_1 = \lambda_N = 0.1$, with $\alpha = 1.25$ and $N = 10$. The error metric is the same as in the preceding subsections. All results were averaged across 10 runs.

\begin{figure}[ht]
\centering
\begin{tabular}{ll}
\includegraphics[width=0.49\columnwidth]{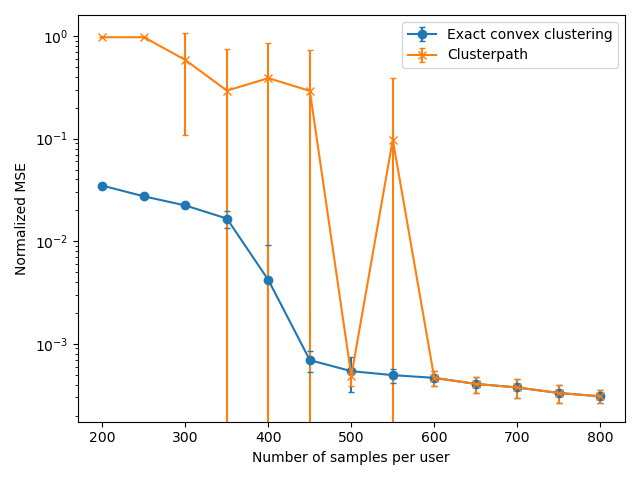}
&
\includegraphics[width=0.49\columnwidth]{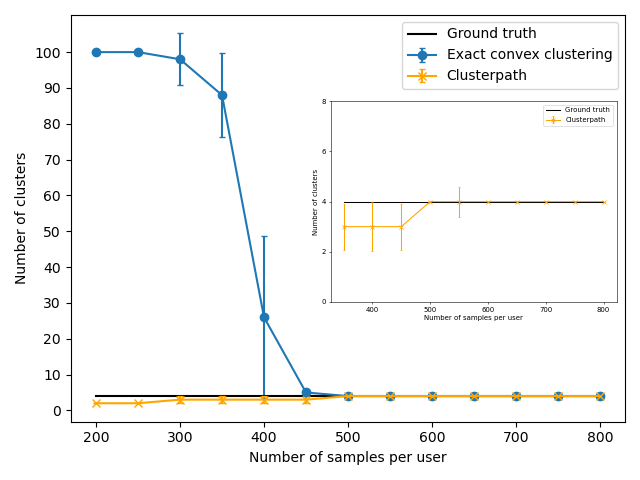}
\end{tabular}
\caption{\emph{Left}: Performance of the proposed methods for linear regression, versus the number of samples available per user. We can see that the method that uses clusterpath matches the performance of the exact convex clustering method, for sufficiently large sample size. \emph{Right}: Number of clusters produced by the two approaches to convex clustering for linear regression, versus the number of samples available per user. We can see that clusterpath matches the performance of exact convex clustering, for sufficiently large sample size.}
\label{fig:MSE_clusterpath}
\end{figure}

Figure~\ref{fig:MSE_clusterpath} presents the performance of the two variants of the proposed method, using linear regression models. The left figure presents the MSE versus the number of samples, while the right figure presents the number of clusters produced by convex clustering versus the number of samples. On $y$-axis in the left figure we plot the averaged normalized MSE, while on the $x$-axis we present the number of samples $n$ available to each user. We can again see that, for a small number of samples (less than $600$), clusterpath performs worse than the exact convex clustering method, that chooses $\lambda$ based on~\eqref{eq:rec_cond}. For larger values of $n$ ($n \geq 600$) we see that clusterpath consistently matches the performance of the exact convex clustering method, showing it achieves the order-optimal MSE rate. On $y$-axis in the right figure we plot the number of clusters produced by the convex clustering algorithm, while on $x$-axis we again plot the number of samples $n$. For smaller sample sizes, clusterpath produces less than the true $K = 4$ clusters, while the exact method produces roughly $m$ clusters, i.e., assigns each user to an individual cluster. These results can be explained by the fact that for small $n$, the interval~\eqref{eq:rec_cond} is likely to be empty, i.e., the upper bound that we take for our value of $\lambda$ smaller than the lower bound, thus creating many clusters. On the other hand, clusterpath always checks the recovery conditions based on the produced clustering, in step 2.(a). Convex clustering is known to have coarsening guarantees, i.e., merging multiple true clusters into one larger, e.g.,~\cite{pmlr-v70-panahi17a,cvx_clust}, in the form of intervals similar to~\eqref{eq:rec_cond}. The coarsening intervals are larger than the corresponding intervals corresponding to the ground truth. Therefore, for smaller values of $n$, clusterpath is likely to produce a coarsening. As $n$ increases, we see that clusterpath is producing a larger number of clusters, eventually stabilizing on $K = 4$ clusters for $n \geq 600$.

\subsection{Comparison with IFCA}\label{subsec:IFCA}

In this subsection we compare the performance of IFCA and ODCL-CC on a synthetic dataset. We consider a linear regression problem, using the same setup as in Subsection~\ref{subsec:clusterpath}, with $K = 4$. 

Since we know the population optima, we can initialize the models of IFCA sufficiently close to the true population optima for each cluster. We initialize each model to be a random vector that is at least $\frac{1}{3}D$ close to the true population optima and at least $\frac{1}{5}D$ away from the true population optima, i.e., for $\theta_k^0 \in \mathbb{R}^d$ the initialization of IFCA, we have $\frac{D}{5} \leq \|\theta^0_k - \theta^\star_k\| \leq \frac{D}{3}$, for all $k \in [K]$. In our experiments, performing a different (random) initialization, we noted that the IFCA performance of IFCA significantly deteriorates if the initialization is not carefully tuned. In contrast, our method requires no specific initialization, hence offering another practical advantage over IFCA. Since the initialization is sufficiently close to the true population optima, we performed IFCA with model averaging (option 2 in their paper), where each user updates the received model for $\tau > 1$ local steps and the server averages the received models, corresponding to the same clusters. We chose $3$ different step-sizes for IFCA, $\alpha = 0.1$, $\alpha = 0.01$ and $\alpha = 0.05$. For our method, we chose the value of $\lambda$ as in Subsection~\ref{subsec:logistic}. The results are averaged across 10 runs.

\begin{figure}[ht]
\centering
\begin{tabular}{ll}
\includegraphics[width=0.49\columnwidth]{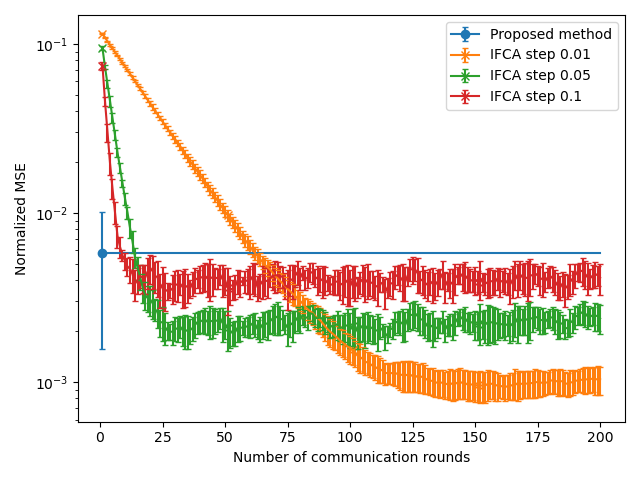}
&
\includegraphics[width=0.49\columnwidth]{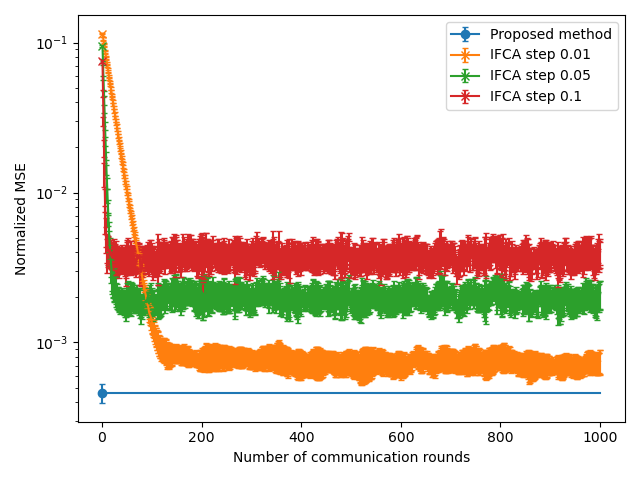}
\end{tabular}
\caption{\emph{Left}: Performance of the proposed method and IFCA for $n = 400$ samples available to each user. We can see that IFCA matches the MSE of our method in a few rounds of communication, in the regime where our method is not order-optimal. \emph{Right}: Performance of the proposed method and IFCA for $n = 600$ samples available to each user. We can see that, while IFCA converges relatively fast (about 100 communication rounds for highest precision), it is unable to achieve the same order-optimal performance in the regime where our method is order-optimal. On the other hand, our method achieves the order optimal-performance in a single communication round.}
\label{fig:IFCA_linreg}
\end{figure}

Figure~\ref{fig:IFCA_linreg} presents the performance of our method and IFCA, using linear regression models. The left figure presents the performance when $n = 400$ samples is available to each user, while the right figure presents the performance when $n = 600$ samples is available to each user. On $y$-axis in both figures we present the normalized MSE, while on $x$-axis we present the number of communication rounds executed along the algorithm iterations. Here, by one communication round we mean the initial server-user broadcast and the returning user-server communication performed during one iteration of the algorithm (note that the order of communications is not important, e.g., our algorithm first has a user-server communication and then a sever-user one, while each communication round of IFCA starts with a server-user communication and ends with a user-server communication). For our method, the MSE is achieved in a single communication round and remains constant (clearly, our algorithm stops after one communication round, we just plot the horizontal line to facilitate the comparison). We can see on the left figure that, for $n = 400$, IFCA requires around $20$ communication rounds to match the performance of our method, for step-sizes $0.1$ and $0.05$, with the performance not improving significantly with each additional communication round. For the step-size $0.01$, IFCA requires around $70$ communication rounds to match the performance of our method, while improving the overall performance by an order of magnitude after approximately $120$ communication rounds. We note that $n = 400$ represents the phase-transitional regime, i.e., a regime in which our method has moved away from clustering each user to a single cluster, but is not yet order-optimal. On the other hand, from the right figure we can see that, for $n = 600$, even after $1000$ communication rounds, IFCA is unable to match the performance of our method, achieved in a single communication round. We note that $n = 600$ is the order-optimal regime for our method, i.e., it is above the threshold shown in Theorem~\ref{thm:main}, therefore representing the regime in which all the communication beyond the first round is redundant, as verified by the right figure.  

\section{Proof of Lemma~\ref{lm:duchi}}\label{app:lemma}

In this section we prove Lemma~\ref{lm:duchi}. The proof follows the idea in~\cite{zhang-duchi}, with some important distinctions stemming from the different assumptions used in this paper. For the sake of completeness, we provide the full proof. For ease of exposition, we assume a single cluster scenario and drop the subscripts $k \in [K]$.  

\begin{proof}[Proof of Lemma~\ref{lm:duchi}]
We begin with the following decomposition
\begin{equation}\label{eq:init_decomposition}
\begin{aligned}
    \mathbb{E}\|\overline{\theta} - \theta^\star\|^2 &= \frac{1}{m^2}\sum_{i = 1}^m\mathbb{E}\|\theta_i - \theta^\star\|^2 + \frac{1}{m^2}\sum_{i \neq j}\mathbb{E}\langle \theta_i - \theta^\star , \theta_j - \theta^\star \rangle \\ &= \frac{1}{m}\mathbb{E}\|\theta_1 - \theta^\star\|^2 + \frac{m - 1}{m}\|\mathbb{E}\theta_1 - \theta^\star\|^2 \leq \frac{1}{m}\mathbb{E}\|\theta_1 - \theta^\star\|^2 + \|\mathbb{E}\theta_1 - \theta^\star\|^2, 
\end{aligned}
\end{equation} where we used the fact that the samples assigned to each user are IID, hence the models $\theta_i$, $i \in [m]$, are statistically independent and identically distributed. Therefore, we want to show that the quantities $\mathbb{E}\|\theta_1 - \theta^\star\|^2$ and $\|\mathbb{E}\theta_1 - \theta^\star\|^2$ decay sufficiently fast in terms of the number of samples $n$. To that end, define $\delta_{\rho} = \min\{\rho, \frac{\rho\mu}{2H} \}$. Consider the events
\begin{align*}
    E_1 &= \left\{\omega: \: \|\nabla f(\theta^\star) \| \leq \frac{(1 - \rho)\mu\delta_\rho}{2} \right\}, \\
    E_2 &= \left\{\omega: \: \|\nabla^2 f(\theta^\star) - \nabla^2 F(\theta^\star)\| \leq \frac{\rho\mu}{2} \right\},
\end{align*} and let $E = E_1 \cap E_2$. We then have the following result.

\begin{lemma}
    On the event $E$ and set $\{\theta: \: \|\theta - \theta^\star\| \leq \delta_\rho \}$, the following holds
    \begin{equation*}
        \|\theta_1 - \theta^\star \| \leq \frac{2\|\nabla f(\theta^\star)\|}{(1 - \rho)\mu} \text{ and } \nabla^2 f(\theta) \geq (1 - \rho)\mu I.
    \end{equation*}
\end{lemma}
\begin{proof}
    First, we show the strong convexity of $f$ on $\{\theta: \: \|\theta - \theta^\star\| \leq \delta_\rho \} \subseteq U$. To that end, for any $\gamma$ from the said set, we have
    \begin{align*}
        \|\nabla^2 f(\gamma) - \nabla^2 F(\theta^\star) \| &\leq \|\nabla^2 f(\gamma) - \nabla^2 f(\theta^\star)\| + \|\nabla^2 f(\theta^\star) -\nabla^2 F(\theta^\star) \| \\ &\leq H\|\gamma - \theta^\star\| + \frac{\rho\mu}{2} \leq H\delta_\rho + \frac{\rho\mu}{2} \leq \rho\mu,
    \end{align*} which implies
    \begin{equation*}
        \nabla^2 f(\gamma) \succeq \nabla^2 F(\theta^\star) - \rho\mu I \succeq (1 - \rho)\mu I.
    \end{equation*} Next, let $\theta^\prime \in \Theta$. Clearly, if $\theta^\prime \in \{\theta: \: \|\theta - \theta^\star\| \leq \delta_\rho \}$, using strong convexity of $f$, we have
    \begin{equation*}
        f(\theta^\prime) \geq f(\theta^\star) + \nabla f(\theta^\star)^\top(\theta^\prime - \theta^\star) + \frac{(1 - \rho)\mu}{2}\|\theta^\prime - \theta^\star\|^2.
    \end{equation*} Otherwise, let $\theta^\prime \in \Theta$ such that $\|\theta^\prime - \theta^\star\| > \delta_{\rho}$. Let $\gamma$ be such that $\|\gamma - \theta^\star\| = \delta_\rho$ and is on the line segment between $\theta^\star$ and $\theta^\prime$, i.e., $\gamma = (1 - k)\theta^\star + k\theta^\prime$, for some $k \in (0,1)$. It can be readily verified that $k = \frac{\delta_\rho}{\|\theta^\prime - \theta^\star\|}$. Using strong convexity of $f$, we then have
    \begin{align*}
        (1 - k)f(\theta^\star) + kf(\theta^\prime) \geq f(\gamma) \geq f(\theta^\star)^\top(\gamma - \theta^\star) + \nabla f(\theta^\star) + \frac{(1 - \rho)\mu}{2}\|\theta^\star - \gamma \|^2, 
    \end{align*} which, after rearranging, gives
    \begin{align*}
        f(\theta^\prime) &\geq f(\theta^\star) + \nabla f(\theta^\star)^\top(\theta^\prime - \theta^\star) + \frac{(1 - \rho)\mu k}{2}\|\theta^\star - \theta^\prime \|^2 \\ &> f(\theta^\star) + \nabla f(\theta^\star)^\top(\theta^\prime - \theta^\star) + \frac{(1 - \rho)\mu \delta_\rho^2}{2}, 
    \end{align*} where we used the definition of $k$, as well as $\|\theta^\prime - \theta^\star\| > \delta_\rho$. Combining both inequalities, we then have, for any $\theta^\prime \in \Theta$
    \begin{equation*}
        f(\theta^\prime) \geq f(\theta^\star) + \nabla f(\theta^\star)^\top(\theta^\prime - \theta^\star) + \frac{(1 - \rho)\mu}{2}\min\left\{\|\theta^\star - \theta^\prime \|^2,\delta_\rho^2\right\}.
    \end{equation*} Setting $\theta^\prime = \theta_1$ and rearranging, we then get
    \begin{equation*}
        \|\nabla f(\theta^\star)\|\|\theta_1 - \theta^\star\| \geq \nabla f(\theta^\star)^\top(\theta_1 - \theta^\star) + f(\theta_1) - f(\theta^\star) \geq \frac{(1 - \rho)\mu}{2}\min\left\{\|\theta^\star - \theta_1 \|^2,\delta_\rho^2\right\},
    \end{equation*} which implies
    \begin{equation*}
        \frac{2\|\nabla f(\theta^\star)\|}{(1 - \rho)\mu} \geq \min\left\{\|\theta^\star - \theta_1 \|,\frac{\delta_\rho^2}{\|\theta^\star - \theta_1\|}\right\}.
    \end{equation*} Assume that $\delta_\rho < \|\theta^\star - \theta_1 \|$. In that case we have
    \begin{equation*}
        \frac{2\|\nabla f(\theta^\star)\|}{(1 - \rho)\mu} \geq \frac{\delta_\rho^2}{\|\theta^\star - \theta_1\|} > \delta_\rho.
    \end{equation*} Since event $E$ holds, it then follows that 
    \begin{equation*}
        \delta_\rho \geq \frac{2\|\nabla f(\theta^\star)\|}{(1 - \rho)\mu} \geq \frac{\delta_\rho^2}{\|\theta^\star - \theta_1\|} > \delta_\rho,
    \end{equation*} which is a contradiction, hence we can conclude that it must hold that $\delta_\rho \geq \| \theta_1 - \theta^\star\|$, hence
    \begin{equation*}
        \frac{2\|\nabla f(\theta^\star)\|}{(1 - \rho)\mu} \geq \|\theta^\star - \theta_1 \|,
    \end{equation*} which is what we wanted to prove.    
\end{proof}

Note that on $E$, we have $\|\theta_1 - \theta^\star \| \leq \delta_\rho$, which implies that $\theta_1 = \argmin_{\theta \in \Theta}f(\theta)$ belongs to the region of strong convexity of $f$, hence $\nabla f(\theta_1) = 0$. We can then use the following Taylor's approximation
\begin{align*}
    0 &= \nabla f(\theta_1) = \nabla f(\theta^\star) + \nabla^2 f(\theta^\prime)(\theta_1 - \theta^\star) = \nabla f(\theta^\star) + \nabla^2 f(\theta^\star)(\theta_1 - \theta^\star) + \left(\nabla^2 f(\theta^\prime) - \nabla^2 f(\theta^\star) \right)(\theta_1 - \theta^\star) \\ &= \nabla f(\theta^\star) + \nabla^2 F(\theta^\star)(\theta_1 - \theta^\star) + \left(\nabla^2 f(\theta^\star) - \nabla^2 F(\theta^\star)\right)(\theta_1 - \theta^\star) + \left(\nabla^2 f(\theta^\prime) - \nabla^2 f(\theta^\star) \right)(\theta_1 - \theta^\star).
\end{align*} Defining $\Sigma = \nabla^2 F(\theta^\star)$, $P = \nabla^2 F(\theta^\star) - \nabla^2 f(\theta^\star)$ and $Q = \nabla^2 f(\theta^\star) - \nabla^2 f(\theta^\prime)$, we then get 
\begin{align}\label{eq:decomp}
    \theta_1 - \theta^\star = -\Sigma^{-1}\nabla f(\theta^\star) + \Sigma^{-1}\left(P + Q\right)(\theta_1 - \theta^\star).
\end{align} Next, note that
\begin{align*}
    \mathbb{E}\|\theta_1 - \theta^\star \|^2 &= \mathbb{E}I_{\{E\}}\|\theta_1 - \theta^\star \|^2 + \mathbb{E}I_{\{E^c\}}\|\theta_1 - \theta^\star \|^2 \\ &\leq \mathbb{E}\|-\Sigma^{-1}\nabla f(\theta^\star) + \Sigma^{-1}(P + Q)(\theta_1 - \theta^\star)\|^2 + 4R^2\mathbb{P}(E^c).
\end{align*} We then have
\begin{align*}
    \mathbb{P}(E^c) = \mathbb{P}(E_1^c \cup E_2^c) \leq \mathbb{P}(E_1^c) + \mathbb{P}(E_2^c) = \mathbb{P}\left(\|\nabla f(\theta^\star) \| > \frac{(1 - \rho)\mu\delta_\rho}{2}\right) + \mathbb{P}\left(\|\nabla^2 f(\theta^\star) - \nabla^2 F(\theta^\star) \| > \frac{\rho\mu}{2}\right). 
\end{align*} Since $\|\nabla \ell(\theta^\star;x) \| \leq N$ and $\| \nabla^2\ell(\theta^\star;x)\| \leq L$, according to Assumptions~\ref{asmpt:loss_function} and~\ref{asmpt:unif_bdd_ptws}, it then follows that both $\nabla f(\theta^\star)$ and $\nabla^2 f(\theta^\star) - \nabla^2 F(\theta^\star)$ are zero-mean sub-Gaussian random vectors (matrices), hence we have
\begin{align*}
    \mathbb{P}\left(\|\nabla f(\theta^\star) \| > \frac{(1 - \rho)\mu\delta_\rho}{2}\right) &\leq 2d\exp\left(-\frac{n(1 - \rho)^2\mu^2\delta_\rho^2}{8c_1N^2}\right), \\
    \mathbb{P}\left(\|\nabla^2 f(\theta^\star) - \nabla^2 F(\theta^\star) \| > \frac{\rho\mu}{2}\right) &\leq 2d\exp\left(-\frac{n\rho^2\mu^2}{8c_2L^2} \right),
\end{align*} where $c_1,c_2 > 0$ are some global constants. Define $\alpha = \min\left\{\frac{(1 - \rho)^2\mu^2\delta_\rho^2}{8c_1N^2}, \frac{\rho^2\mu^2}{8c_2L^2}\right\}$. We then have
\begin{equation*}
    \mathbb{P}(E^c) \leq 4d\exp\left(-\alpha n\right).
\end{equation*} Next, we want to provide a bound on the remaining term, namely $\mathbb{E}\|-\Sigma^{-1}\nabla f(\theta^\star) + \Sigma^{-1}(P + Q)(\theta_1 - \theta^\star)\|^2$. We have
\begin{align*}
    \mathbb{E}\|-\Sigma^{-1}\nabla f(\theta^\star) &+ \Sigma^{-1}(P + Q)(\theta_1 - \theta^\star)\|^2 \leq 2\mathbb{E}\|\Sigma^{-1}\nabla f(\theta^\star)\|^2  + 2\mathbb{E}\|\Sigma^{-1}(P + Q)(\theta_1 - \theta^\star)\|^2 \\ &\leq 2\mathbb{E}\|\Sigma^{-1}\nabla f(\theta^\star)\|^2  + 4\|\Sigma^{-1} \|^2\left(\mathbb{E}\|P(\theta_1 - \theta^\star) \|^2 + \mathbb{E}\|Q(\theta_1 - \theta^\star)\|^2\right) \\ &\leq 2\mathbb{E}\|\Sigma^{-1}\nabla f(\theta^\star)\|^2  + 4\|\Sigma^{-1} \|^2\left(\sqrt{\mathbb{E}\|P\|^4\mathbb{E}\|\theta_1 - \theta^\star\|^4} + \sqrt{\mathbb{E}\|Q\|^4\mathbb{E}\|\theta_1 - \theta^\star\|^4}\right).
\end{align*} We next focus on bounding $\mathbb{E}\|P \|^4$, $\mathbb{E}\|Q \|^4$ and $\mathbb{E}\|\theta_1 - \theta^\star \|^4$. Since we are conditioning on the event $E$, we have
\begin{align*}
    \mathbb{E}\|\theta_1 - \theta^\star \|^4 \leq \frac{2^4\mathbb{E}\|\nabla f(\theta^\star)\|^4}{(1 - \rho)^4\mu^4}.
\end{align*} We then want to bound $\mathbb{E}\|\nabla f(\theta^\star)\|^4$. To that end, we have
\begin{equation}\label{eq:fourth-moment}
    \mathbb{E}\|\nabla f(\theta^\star) \|^4 = \mathbb{E}\left[\|\nabla f(\theta^\star) \| \pm \mathbb{E}\|\nabla f(\theta^\star)\|\right]^4 \leq 8\mathbb{E}\left|\|\nabla f(\theta^\star) \| - \mathbb{E}\|\nabla f(\theta^\star)\|\right|^4 + 8\left[\mathbb{E}\|\nabla f(\theta^\star)\|\right]^4.
\end{equation} The second term can be bounded as follows
\begin{equation*}
    \left[\mathbb{E}\|\nabla f(\theta^\star)\|\right]^4 \leq \left[\mathbb{E}\|\nabla f(\theta^\star)\|^2\right]^2 \leq \frac{N^4}{n^2},
\end{equation*} where in the first inequality we used Jensen's inequality, while the second inequality follows from the fact that $\nabla f(\theta^\star)$ is the average of $n$ zero-mean IID bounded random variables. To bound the first term, we use the following lemma:
\begin{lemma}\label{lm:bound_abs_diff}
    Let $k \geq 2$ and $X_i$, $i = 1,\ldots,n$ be a sequence of independent random vectors in a separable Banach space with norm $\| \cdot \|$ and $\mathbb{E}[\|X_i\|^k] < \infty$. There exists a finite constant $C_k$ such that
    \begin{equation*}
        \mathbb{E}\left|\left\|\sum_{i = 1}^nX_i \right\| - \mathbb{E}\left\|\sum_{i = 1}^n X_i\right\|\right|^k \leq C_k\left[\left(\sum_{i = 1}^n\mathbb{E}\|X_i\|^2 \right)^\frac{k}{2} + \sum_{i = 1}^n\mathbb{E}\|X_i\|^k \right].
    \end{equation*}
\end{lemma} Applying Lemma~\ref{lm:bound_abs_diff} to the first term in~\ref{eq:fourth-moment}, we then get
\begin{align*}
    \mathbb{E}\left|\|\nabla f(\theta^\star) \| - \mathbb{E}\|\nabla f(\theta^\star)\|\right\|^4 &\leq C_4\left[\left(\frac{1}{n^2}\sum_{i = 1}^n\mathbb{E}\|\nabla \ell(\theta^\star;X_i)\|^2 \right)^2 + \frac{1}{n^4}\sum_{i = 1}^n\mathbb{E}\|\nabla \ell(\theta^\star;X_i)\|^4 \right] \leq \frac{C_4N^2}{n^2} + \frac{C_4N^4}{n^3}.
\end{align*} Combining, we get that
\begin{equation*}
    \mathbb{E}\|\nabla f(\theta^\star)\|^4 \leq \frac{8(C_4 + N^2)N^2}{n^2} + \frac{8C_4N^4}{n^3} = \mathcal{O}\left(\frac{N^2(1 + N^2)}{n^2} \right).
\end{equation*} Without the loss of generality, we assume $N \geq 1$, (otherwise, we can always use a more loose upper bound $\widetilde{N} = N + 1$) which gives that, on $E$
\begin{equation*}
    \mathbb{E}\|\theta_1 - \theta^\star \|^4 = \mathcal{O}\left(\frac{N^4}{(1 - \rho)^4\mu^4n^2} \right).
\end{equation*} Note that this immediately gives a global bound on $\mathbb{E}\|\theta_1 - \theta^\star \|^4$, as follows
\begin{align*}
    \mathbb{E}\|\theta_1 - \theta^\star \|^4 &= \mathbb{E}I_{E}\|\theta_1 - \theta^\star \|^4 + \mathbb{E}I_{E^c}\|\theta_1 - \theta^\star \|^4 \leq  \mathcal{O}\left(\frac{N^4}{(1 - \rho)^4\mu^4n^2} \right) + 2^4R^4\mathbb{P}(E^c) \\ &= \mathcal{O}\left(\frac{N^4}{(1 - \rho)^4\mu^4n^2}\right) + 2^4CR^4d\exp(-\alpha n) = \mathcal{O}\left(\frac{\beta}{n^2}\right),
\end{align*} where $\beta = \nicefrac{N^4}{(1 - \rho)^4\mu^4} + \nicefrac{dR^4}{\alpha}$. Next, we want to bound $\mathbb{E}\|P\|^4$, where we recall that $P = \nabla^2 F(\theta^\star) - \nabla^2 f(\theta^\star)$. To that end, we use the following intermediary result.
\begin{lemma}\label{lm:matrices-fourth-mom}
    Let $k \geq 2$ and $X_i \in \mathbb{R}^{d\times d}$ be independent and symmetrically distributed Hermitian matrices. Then
    \begin{equation*}
        \left(\mathbb{E}\left\| \sum_{i=1}^n X_i\right\|^k\right)^\frac{1}{k} \leq \sqrt{2e\log d}\left\|\left(\sum_{i = 1}^n\mathbb{E}X_i^2\right)^\frac{1}{2}\right\| +2e\log d \left(\mathbb{E}\max_{i = 1,\ldots,n} \|X_i\|^k\right)^\frac{1}{k}.
    \end{equation*}
\end{lemma} Define $Z_i = \frac{1}{n}\left(\nabla^2\ell(\theta^\star;X_i) - \nabla^2 F(\theta^\star)\right)$ and notice that $\varepsilon_iZ_i$ are symmetrically distributed matrices, where $\varepsilon_i \in \{-1,1 \}$ are IID Rademacher random variables, independent of $Z_i$. Therefore, applying a standard symmetrization argument and Lemma~\ref{lm:matrices-fourth-mom} (since $\varepsilon_iZ_i$'s are Hermitian and symmetrically distributed), we get
\begin{align*}
    \mathbb{E}\|P\|^4 &= \mathbb{E}\left\|\sum_{i = 1}^nZ_i\right\|^4 \leq 2^4\mathbb{E}\left\|\sum_{i = 1}^n\varepsilon_iZ_i \right\|^4 \\ &\leq \left(5\sqrt{\log d}\left\|\left(\sum_{i = 1}^n\mathbb{E}(\varepsilon_iZ_i)^2\right)^\frac{1}{2}\right\| +4e\log d \left(\mathbb{E}\max_{i} \left\|\varepsilon_iZ_i\right\|^4\right)^\frac{1}{4}\right)^4 \\ &= \left(5\sqrt{\log d}\left\|\left(\sum_{i = 1}^n\mathbb{E}Z_i\right)^\frac{1}{2}\right\| +4e\log d \left(\mathbb{E}\max_{i} \left\|Z_i\right\|^4\right)^\frac{1}{4}\right)^4.
\end{align*} Since $\|\nabla^2 \ell(\theta^\star;X_i)\| \leq L$, we then have $\|\nabla^2 \ell(\theta^\star;X_i) - \nabla^2 F(\theta^\star)\| \leq 2L$, for all $i$, therefore
\begin{align*}
    \left\|\left(\sum_{i = 1}^n\mathbb{E}Z_i^2\right)^\frac{1}{2} \right\| &= \frac{1}{\sqrt{n}}\left\|\left(\mathbb{E}\left(\nabla^2\ell(\theta^\star;X_1) - \nabla^2F(\theta^\star)\right)^2\right)^\frac{1}{2} \right\| \\ &\leq \frac{1}{\sqrt{n}}\sqrt{\mathbb{E}\left\|\nabla^2\ell(\theta^\star;X_1) - \nabla^2F(\theta^\star)\right\|^2} \leq \frac{2L}{\sqrt{n}}
\end{align*} where the first equality follows from the fact that $Z_i$'s are IID and the second inequality follows from the fact that $\|A^\frac{1}{2}\| \leq \|A \|^\frac{1}{2}$ for semi-definite matrices and Jensen's inequality. Similarly, we have
\begin{equation*}
    \frac{1}{n^4}\mathbb{E}\max_i\left\|\nabla^2\ell(\theta^\star;X_i) - \nabla^2F(\theta^\star) \right\|^4 \leq \frac{64L^4}{n^4}.
\end{equation*} Combining everything, we then get
\begin{equation*}
    \mathbb{E}\|P\|^4 \leq \left(\frac{10L\sqrt{\log d}}{\sqrt{n}} + \frac{8eL\log d}{n} \right)^4 = \mathcal{O}\left(\frac{L^4\log^4d}{n^2} \right).
\end{equation*} What remains is to bound $\mathbb{E}\|Q\|^4$, where we recall that $Q = \nabla^2 f(\theta^\star) - \nabla^2 f(\theta^\prime)$ and $\theta^\prime = (1 - k)\theta^\star + k\theta_1$, for some $k \in (0,1)$. To that end, we have
\begin{align*}
    \mathbb{E}\|Q \|^4 &= \mathbb{E}\left\|\frac{1}{n}\sum_{i = 1}^n\left(\nabla^2\ell(\theta^\prime;X_i) - \nabla^2\ell(\theta^\star;X_i) \right) \right\|^4 \leq \frac{1}{n}\sum_{i = 1}^n\mathbb{E}\|\nabla^2 \ell(\theta^\prime;X_i) - \nabla^2\ell(\theta^\star;X_i) \|^4 \\ & \leq \mathbb{E}\|\nabla^2 \ell(\theta^\prime;X_1) - \nabla^2\ell(\theta^\star;X_1) \|^4 \leq H^4\mathbb{E}\|\theta^\prime - \theta^\star \|^4 \leq H^4\mathbb{E}\|\theta_1 - \theta^\star\|^4,
\end{align*} where we used Jensen's inequality in the first, the fact that $X_i$'s are IID in the second, the fact that $\|\theta^\prime - \theta^\star\| = k\|\theta_1 - \theta^\star\| \leq \delta_\rho$, which implies that $\theta^\prime \in U$, and Assumption~\ref{asmpt:loc_behaviour} in the third inequality. Combining everything, gives
\begin{align*}
    \mathbb{E}&\|-\Sigma^{-1}\nabla f(\theta^\star) + \Sigma^{-1}(P + Q)(\theta_1 - \theta^\star)\|^2 \\ &\leq 2\mathbb{E}\|\Sigma^{-1}\nabla f(\theta^\star)\|^2 + 4\|\Sigma^{-1} \|^2\left(\sqrt{\mathbb{E}\|P\|^4\mathbb{E}\|\theta_1 - \theta^\star\|^4} + H^2\mathbb{E}\|\theta_1 - \theta^\star\|^4\right) \\ &\leq 2\|\Sigma^{-1}\|^2\mathbb{E}\|\nabla f(\theta^\star)\|^2 + 4\|\Sigma^{-1} \|^2\left(\sqrt{\mathcal{O}\left(\frac{L^4\log^4d}{n^2} \right)\mathbb{E}\|\theta_1 - \theta^\star\|^4} + H^2\mathbb{E}\|\theta_1 - \theta^\star\|^4\right) \\ &\leq 2\|\Sigma^{-1}\|\mathcal{O}\left(\frac{N^2}{n} \right) + 4\|\Sigma^{-1} \|^2\left(\sqrt{\mathcal{O}\left(\frac{L^4\log^4d}{n^2}\right)\mathcal{O}\left(\frac{\beta}{n^2} \right)} + \mathcal{O}\left(\frac{H^2\beta}{n^2}\right)\right) \\ &\leq \|\Sigma^{-1}\|^2\mathcal{O}\left(\frac{N^2}{n} + \frac{\beta^{\nicefrac{1}{2}}L^2\log^2 d + H^2\beta}{n^2} \right).
\end{align*} Finally, since $\|\Sigma^{-1}\|^2 \leq \frac{1}{\mu^2}$, we have
\begin{equation}\label{eq:bound-mse-user}
    \mathbb{E}\|\theta_1 - \theta^\star\|^2 = \mathcal{O}\left(\frac{N^2}{n\mu^2} + \frac{\beta^{\nicefrac{1}{2}}L^2\log^2 d + H^2\beta}{n^2\mu^2} \right).
\end{equation} We now move on to bounding the second term of interest, i.e., $\|\mathbb{E}\theta_1 - \theta^\star\|^2$. When $E$ occurs, we can recursively apply~\eqref{eq:decomp}, to get
\begin{align*}
    \theta_1 - \theta^\star &= -\Sigma^{-1}\nabla f(\theta^\star) + \Sigma^{-1}\left(P + Q\right)(\theta_1 - \theta^\star) \\ &= -\Sigma^{-1}\nabla f(\theta^\star) + \Sigma^{-1}\left(P + Q\right)(-\Sigma^{-1}\nabla f(\theta^\star) + \Sigma^{-1}\left(P + Q\right)(\theta_1 - \theta^\star)).
\end{align*} Define $v = -\Sigma^{-1}\nabla f(\theta^\star) + \Sigma^{-1}\left(P + Q\right)(-\Sigma^{-1}\nabla f(\theta^\star) + \Sigma^{-1}\left(P + Q\right)(\theta_1 - \theta^\star))$, to get
\begin{equation*}
    \theta_1 - \theta^\star = I_{\{E\}}v + I_{\{E^c\}}(\theta_1 - \theta^\star) = v + I_{\{E^c\}}(\theta_1 - \theta^\star - v),
\end{equation*} which implies
\begin{equation*}
    \|\mathbb{E}\theta_1 - \theta^\star \| \leq \|\mathbb{E}v\| + \mathbb{E}I_{E^c}\|\theta_1 - \theta^\star - v\| \leq \|\mathbb{E}v\| + 2R\mathbb{P}(E^c) + \mathbb{E}I_{E^c}\|v\|.
\end{equation*} Next, we look at $\|\mathbb{E}v\|$. Noticing that $\mathbb{E}\nabla f(\theta^\star) = 0$, gives
\begin{align*}
    \|\mathbb{E}v\| &= \|\Sigma^{-1}\mathbb{E}\left(P + Q\right)(-\Sigma^{-1}\nabla f(\theta^\star) + \Sigma^{-1}\left(P + Q\right)(\theta_1 - \theta^\star))\| \\ &\leq
    \mathbb{E}\|\Sigma^{-1}\left(P + Q\right)(-\Sigma^{-1}\nabla f(\theta^\star) + \Sigma^{-1}\left(P + Q\right)(\theta_1 - \theta^\star))\| \\ &\leq \mathbb{E}\|\Sigma^{-1}\left(P + Q\right)\Sigma^{-1}\nabla f(\theta^\star)\| + \mathbb{E}\|\Sigma^{-1}\left(P + Q\right)\Sigma^{-1}\left(P + Q\right)(\theta_1 - \theta^\star))\| \\ &\leq \|\Sigma^{-1}\|^2\left( \mathbb{E}\|P + Q\|\|\nabla f(\theta^\star) \| + \mathbb{E}\|P + Q\|^2\|\theta_1 - \theta^\star\| \right) \\ &\leq \|\Sigma^{-1}\|^2\left( \sqrt{\mathbb{E}\|P + Q\|^2\mathbb{E}\|\nabla f(\theta^\star) \|^2} + \sqrt{\mathbb{E}\|P + Q\|^4\mathbb{E}\|\theta_1 - \theta^\star\|^2} \right).
\end{align*} Recalling from the above results that
\begin{align*}
    \mathbb{E}\|\theta_1 - \theta^\star\|^4 &= \mathcal{O}\left(\frac{\beta}{n^2}\right), \\
    \mathbb{E}\|P\|^4 &= \mathcal{O}\left(\frac{L^4\log^4d}{n^2} \right), \\
    \mathbb{E}\|Q\|^4 &= \mathcal{O}\left(\frac{H^4\beta}{n^2}\right), \\
    \mathbb{E}\|\nabla f(\theta^\star)\|^4 &= \mathcal{O}\left(\frac{N^4}{n^2}\right),
\end{align*} we then get
\begin{align*}
    \|\mathbb{E}v\|^2 &\leq 2\|\Sigma^{-1} \|^4\left(\mathbb{E}\|P + Q\|^2\mathbb{E}\|\nabla f(\theta^\star) \|^2 + \mathbb{E}\|P + Q\|^4\mathbb{E}\|\theta_1 - \theta^\star\|^2 \right) \\ &\leq 4\|\Sigma^{-1} \|^4\left(\mathbb{E}\|\nabla f(\theta^\star) \|^2\left(\mathbb{E}\|P\|^2 + \mathbb{E}\|Q\|^2\right) + 4\mathbb{E}\|\theta_1 - \theta^\star\|^2\left(\mathbb{E}\|P\|^4 + \mathbb{E}\|Q\|^4 \right)\right) \\ &= \mathcal{O}\left(\frac{N^2L^2\log^2 d + \beta^{\nicefrac{1}{2}}H^2N^2}{n^2\mu^4} +\frac{\beta^{\nicefrac{1}{2}}L^4\log^4 d + \beta^{\nicefrac{3}{2}}H^4}{n^3} \right).
\end{align*} Next, we look at the remaining term, $\mathbb{E}I_{E^c}\|v\|$. We then have
\begin{align*}
    \mathbb{E}I_{E^c}\|v \| \leq \sqrt{\mathbb{P}(E^c)\mathbb{E}\|v \|^2} \leq 2\sqrt{d}\exp(-\nicefrac{\alpha n}{2})\sqrt{\mathbb{E}\|v\|^2}.
\end{align*} We now bound $\mathbb{E}\|v\|^2$. To that end, we have
\begin{align*}
    \mathbb{E}\|v\|^2 &= \mathbb{E}\|-\Sigma^{-1}\nabla f(\theta^\star) + \Sigma^{-1}\left(P + Q\right)(-\Sigma^{-1}\nabla f(\theta^\star) + \Sigma^{-1}\left(P + Q\right)(\theta_1 - \theta^\star)) \|^2 \\ &\leq 2\|\Sigma^{-1}\|^2\mathbb{E}\|\nabla f(\theta^\star)\|^2 + 2\mathbb{E}\|\Sigma^{-1}\left(P + Q\right)(-\Sigma^{-1}\nabla f(\theta^\star) + \Sigma^{-1}\left(P + Q\right)(\theta_1 - \theta^\star)) \|^2 \\ &\leq 2\|\Sigma^{-1}\|^2\left(\mathbb{E}\|\nabla f(\theta^\star)\|^2 + 2\|\Sigma^{-1}\|^2\left[\mathbb{E}\|P + Q\|^2\|\nabla f(\theta^\star) \|^2 + \mathbb{E}\|P + Q \|^4\|\theta_1 - \theta^\star\|^2\right]\right).
\end{align*} Applying Holder's inequality and the previously established bounds (note that a bound on $\mathbb{E}\|P\|^8$, $\mathbb{E}\|Q\|^8$ can be established similarly to the one of $\mathbb{E}\|P\|^4$, $\mathbb{E}\|Q\|^4$) gives
\begin{align*}
    \mathbb{E}\|v\|^2 = \mathcal{O}\left(\frac{N^2}{n\mu^2} + \frac{\sqrt{N^4L^4\log^4d + \beta N^4H^4}}{n^2\mu^4} + \frac{\sqrt{\beta L^8\log^8d + \beta^3H^8}}{n^3\mu^4} \right) = \mathcal{O}\left(\frac{\tau}{n}\right),
\end{align*} where $\tau = \frac{N^2}{\mu^2} + \frac{\sqrt{N^4L^4\log^4d + \beta N^4H^4}}{\mu^4} + \frac{\sqrt{\beta L^8\log^8d + \beta^3H^8}}{\mu^4}$. This implies 
\begin{equation*}
    \mathbb{E}I_{E^c}\|v\| = \mathcal{O}\left(\sqrt{\frac{\tau d}{n}}\exp(-\nicefrac{\alpha n}{2})\right)
\end{equation*} Combining, we finally get that
\begin{equation}\label{eq:norm-expected}
\begin{aligned}
    \|\mathbb{E}\theta_1 - \theta^\star\|^2 = \mathcal{O}\left(\frac{N^2L^2\log^2 d + \beta^{\nicefrac{1}{2}}H^2N^2}{n^2\mu^4} +\frac{\beta^{\nicefrac{1}{2}}L^4\log^4 d + \beta^{\nicefrac{3}{2}}H^4}{n^3} + \left(R^2d + \frac{\tau }{n}\right)d\exp(-\alpha n) \right).
\end{aligned}
\end{equation} Plugging~\eqref{eq:bound-mse-user} and~\eqref{eq:norm-expected} in~\eqref{eq:init_decomposition}, finally gives
\begin{align*}
    \mathbb{E}\|\overline{\theta} - \theta^\star \|^2 = \mathcal{O}\bigg(\frac{N^2}{nm\mu^2} + \frac{\lambda}{n^2} + \frac{G}{n^2m} + \frac{\nu}{n^3} + \left(R^2d + \frac{\tau}{n}\right)d\exp(-\alpha n)\bigg),
\end{align*} where $G = \frac{\beta^{\nicefrac{1}{2}}L^2\log^2 d + H^2\beta}{\mu^2}$, $\lambda = \frac{N^2L^2\log^2 d + \beta^{\nicefrac{1}{2}}H^2N^2}{\mu^4}$, $\nu = \beta^{\nicefrac{1}{2}}L^4\log^4 d + \beta^{\nicefrac{3}{2}}H^4$.
\end{proof}

\end{document}